\documentclass[11pt]{article}

\usepackage[margin=1in]{geometry}

\usepackage[round]{natbib}
\usepackage{amsmath,amssymb,amsthm,mathtools}
\usepackage{booktabs}
\usepackage{graphicx}
\usepackage{microtype}
\usepackage{enumitem}
\usepackage{verbatim}
\usepackage{longtable}
\usepackage{tikz}
\usetikzlibrary{arrows.meta,calc}
\usepackage[dvipsnames]{xcolor}
\usepackage[colorlinks=true,linkcolor=NavyBlue,citecolor=NavyBlue,urlcolor=NavyBlue]{hyperref}
\usepackage{threeparttable}

\newcommand{\apptable}{%
  \centering
  \setlength{\abovecaptionskip}{11pt}%
  \setlength{\belowcaptionskip}{11pt}}

\theoremstyle{plain}
\newtheorem{theorem}{Theorem}[section]

\newtheorem{lemma}[theorem]{Lemma}

\theoremstyle{definition}

\theoremstyle{remark}

\DeclareMathOperator{\tr}{tr}
\DeclareMathOperator{\diag}{diag}
\DeclareMathOperator{\Var}{Var}
\DeclareMathOperator{\Cov}{Cov}

\DeclareMathOperator{\rank}{rank}
\newcommand{\T}{^{\mathsf{T}}}
\newcommand{\R}{\mathbb{R}}
\newcommand{\E}{\mathbb{E}}
\newcommand{\Ez}{\E_z}
\newcommand{\one}{\mathbf{1}}
\newcommand{\vhat}{\widehat{\Var}_N}
\newcommand{\Kt}{\widetilde{K}}

\newcommand{\Rsq}{R^{2}}
\newcommand{\Rmarg}{R^{2}_{\mathrm{marg}}}
\newcommand{\Rreal}{R^{2}_{\mathrm{real}}}
\newcommand{\Remp}{R^{2}_{\mathrm{emp}}}

\makeatletter
\newcommand{\supplementsizes}{%
  \renewcommand\normalsize{%
    \@setfontsize\normalsize\@xpt\@xiipt
    \abovedisplayskip 10\p@ \@plus2\p@ \@minus5\p@
    \abovedisplayshortskip \z@ \@plus3\p@
    \belowdisplayshortskip 6\p@ \@plus3\p@ \@minus3\p@
    \belowdisplayskip \abovedisplayskip}%
  \renewcommand\small{\@setfontsize\small\@ixpt{11}}%
  \renewcommand\footnotesize{\@setfontsize\footnotesize\@viiipt{9.5}}%
  \renewcommand\scriptsize{\@setfontsize\scriptsize\@viipt\@viiipt}%
  \renewcommand\tiny{\@setfontsize\tiny\@vpt\@vipt}%
  \normalsize}
\makeatother

\title{\bfseries R2-GP: Variance-Calibrated Priors for Additive Gaussian Processes}

\usepackage{authblk}

\author[2]{Javier Enrique Aguilar}
\author[1]{Soham Mukherjee}
\author[1,*]{David Kohns}
\affil[1]{Aalto University, Finland}
\affil[2]{TU Dortmund, RC TRUST (UA Ruhr), Germany}
\affil[*]{Corresponding author: \href{mailto:david.kohns@aalto.fi}{\texttt{david.kohns@aalto.fi}}}

\hypersetup{
  pdftitle={R2-GP: Variance-Calibrated Priors for Additive Gaussian Processes},
  pdfauthor={Javier Enrique Aguilar, Soham Mukherjee, David Kohns}}

\date{}

\begin{document}
\maketitle

\begin{abstract}
\noindent
We develop R2-GP, an R2D2-type prior for Gaussian process (GP) generalised additive models (GAMs) that calibrates explained variance ($R^2$) for Matérn covariance kernels. The function variance for standard Matérn kernels depends jointly on the effect magnitude and length-scale.
As such, smooth and rough components with identical amplitudes contribute very different amounts of variation. With independent priors on component amplitudes, this variation results in variance overestimation, especially in higher dimensions.
We overcome this by taking direct control over each component's expected contribution to prior explained variance. We do this by rescaling each component based on its length scale and observed inputs, and modelling each rescaled variance as a share of the function variance induced by a global $R^2$ prior.
This allows R2-GP to scale well to many components. To also scale to large sample sizes, we extend R2-GP to Hilbert space approximation methods, a framework for kernels exploiting their spectral decompositions for fast and efficient inference. Our simulations and real-world case studies demonstrate superior generalisation and parameter recovery over independent variance priors.
\end{abstract}

\section{Introduction}
\label{sec:intro}
Gaussian processes (GPs) are a class of probabilistic methods for modelling non-linear functions. They do so by assuming a multivariate Gaussian prior distribution over a finite set of these functions and are specified through a mean function and a covariance kernel \citep{Rasmussen1995}. Common choices of covariance kernel involve the Mat\'ern class of functions \citep{rasmussen2006gaussian}, defined using a length scale parameter and a magnitude parameter (also called GP marginal variance). The former controls the correlation decay rate between two arbitrary inputs while the latter controls prior variance of the functions at a single input.

GP priors have been specifically used in generalised additive models \citep[GAMs;][]{Hastie1986gams, hastie1990gam} as an approach to model smooth functions \citep{Wood2017gams}. In GAMs, a response $y \in \R^N$ is modelled using multiple covariates $X \in \R^{N\times D}$ such that
\begin{equation}\label{eq:gp-gams}
  y_i = \mu_0 + \sum^D_{d=1}f_d(x_{id}) + \varepsilon_i,
  \qquad \varepsilon_i \sim N(0,\sigma^2),
\end{equation}
for $i = 1,\dots,N$, where $\mu_0$ is the intercept, $\varepsilon_i$ is the residual error and $f_d$ is a GP specific to the $d$-th covariate. %
Additive GPs arise as a special case of the general class of models presented in \citet{duvenaud2011additive}.

In GAMs, modelling a response with many covariates introduces a variance explosion problem familiar from high-dimensional linear regression. If each of the $D$ components receives an independent variance prior with positive mass on non-zero values, the total signal variance accumulates as $D$ grows. Consequently, the implied prior on the proportion of explained variance $R^2$ pathologically concentrates near one. This may lead to overfitting and contradict the modeller's domain knowledge about $R^2$, rendering the model predictively inconsistent \citep{riha2026select}.

Preventing this variance accumulation requires explicitly controlling the model's total variance. In regression contexts, a natural approach is to place a prior on $R^2$ and partition it across the components, analogous to the R2D2 prior for linear models \citep{zhang2022r2d2}. The mapping between a component's magnitude parameter and the observed function variance one wishes to control is however blurred by the simultaneous impact of the length scale and the input design. Because of this entanglement, smooth and rough components with identical magnitudes contribute vastly different amounts to observed function variation (Figure~\ref{fig:variance-gap}).

We overcome this structural issue with our R2-GP prior. By rescaling each component with a normaliser which is derived from its length scale and the input design, R2-GP forces the component's magnitude to exactly match its expected variation across the data. This calibration allows a global $R^2$ prior which is partitioned with a simplex prior that forces components to compete for a share of the variance budget.

A drawback of this normalisation is the computational burden it adds to posterior inference. Since the normaliser depends on the length scale, it must be recomputed at every gradient evaluation when using MCMC methods like Hamiltonian Monte Carlo \citep{neal2011mcmc, hoffman2014nuts}. For an exact GP, this requires $O(N^2)$, added to the standard $O(N^3)$ cost of factorising the exact covariance matrix.

We solve this scalability issue by extending the R2-GP prior to Hilbert space approximation methods (HSGPs) \citep{solin2020hilbert, riutortmayol2023practical,Mukherjee2026moHSGPs}. HSGPs exploit the spectral decomposition of the kernel using basis functions $M$ for large sample scalability and have been extended to GAMs efficiently \citep{Timonen2025GPlong}. With the extension to HSGPs based on $M<<N$ basis functions, computing the normaliser simplifies to a single pre-computation of $O(NM)$ and $O(M)$. %

Code to replicate the experiments is available at \href{https://github.com/davkoh/r2-gp}{https://github.com/davkoh/r2-gp}.

\paragraph{Contributions.}
\label{sec:intro-contributions}
\begin{enumerate}%
  \item We develop R2-GP by extending the R2D2 family of priors to GPs for GAMs. We derive the necessary normalising approach of the kernel such that a prior over the GP's $R^2$ maps analytically to its expected sample variance. A Dirichlet decomposition of $R^2$ into contributions of its GAM components allows its application to high dimensions.
  \item We derive R2-GP prior for the Hilbert space approximation and show that the normaliser admits an exact finite-sample expression requiring $O(NM)$ preparation and $O(M)$ and scales to high dimensions.%
  \item Through simulations and real-world data, we show that the R2-GP prior outperforms independent priors for additive HSGPs in prediction and parameter recovery and sampling efficiency, and is competitive with alternate shrinkage priors.
\end{enumerate}

\paragraph{Related Work.}
\label{sec:intro-related}
Priors placed on $R^2$, since the original \citet{zhang2022r2d2}, have been extended to multilevel models \citep{aguilar2023r2d2m2}, generalised linear mixed models \citep{yanchenko2025r2d2}, autoregressions \citep{kohns2025arr2}, spatial models \citep{yanchenko2024spatial} and grouped predictors \citep{aguilar2025groupr2}. The present paper takes the variance allocation mechanism from that literature and addresses the question it leaves on how to extend this to GPs.%

Closest to our normalisation approach is the standardisation procedure of \citet{ferrari2025standardization}, who scale variance partitioning priors in latent Gaussian models by the variance of a zero-mean-constrained effect over an assumed covariate distribution, evaluated at fixed values of unknown kernel parameters such as the range of a Mat\'ern process. Our normaliser is instead a function of the length scale and is recomputed at every evaluation inside the sampler, which is what makes the $O(M)$ identity of Lemma~\ref{thm:Q} particularly efficient. Penalised complexity priors for Gaussian random fields \citep{simpson2017penalising,fuglstad2019constructing} shrink towards a simpler base model, whereas the normalisation here fixes the units in which the scale parameter is expressed, and the two could in principle be combined.

The remainder of the paper is organised as follows. Section~\ref{sec:r2hsgp} derives the normalised additive prior and its properties, among them the implied prior on explained variance. Section~\ref{sec:simulation} presents our simulation studies and Section~\ref{sec:case_study} demonstrates a real-world application. Proofs and additional results are collected in Appendix.%

\section{Methods}%
\label{sec:r2hsgp}
In our framework, the $d$-th column of the design matrix $X$ holds the observed values $x_{1d},\dots,x_{Nd}$ of covariate $d$. We use  GAM GP prior with independent processes per covariate,
\begin{equation}
  f_d = \alpha_d\, g_d,
  \qquad
  g_d \sim \mathcal{GP}\bigl(0,\, k(\cdot,\cdot\,;\rho_d)\bigr),
  \label{eq:intro-model}
\end{equation}
where $k$ is a stationary kernel of the Mat\'ern family, which includes the squared exponential (SE) kernel as a limiting case, with $k(x,x;\rho_d) = 1$, so that $g_d$ has unit marginal variance. Parameter $\alpha_d^2$ is the \emph{magnitude} of the component, which equals its pointwise \emph{marginal variance}, $\Var(f_d(x)) = \alpha_d^2$ at every input.\footnote{The GP literature often calls $\alpha_d^2$ itself the signal variance \citep{rasmussen2006gaussian}. Section~\ref{sec:intro-gap} shows that the two differ.} In contrast, the \emph{signal variance} is the expected variation that the function exhibits across the inputs, $\E_{f_d}[\Var_x(f_d(x)\mid f_d)]$. We denote by $\E_x$ and $\Var_x$ the expectation and variance operators with respect to $x$. In the following, we first derive the mathematical disconnect between these two variances,  then provide the necessary concepts of the Hilbert space (HS) approximation to finally define an efficient finite-sample normaliser, which we ultimately use to construct the R2-GP prior.

\subsection{Calibrating GPs for Explained Variance}
\label{sec:intro-gap}

In linear regression, placing a prior directly on the proportion of explained variance $\Rsq = \Var(f(x))/\Var(y)$ provides an interpretable way to specify regularisation on variance parameters of regression coefficients. %
For GAM GPs the $R^2$ identity is,
\begin{equation}
  \Rmarg =\frac{\sum_{d=1}^{D}\alpha_d^2}{\sum_{d=1}^{D}\alpha_d^2 + \sigma^2},
  \label{eq:intro-rmarg}
\end{equation}
where a prior $\pi_{R^2}$ maps to prior $\pi_{\alpha_d}$, conditional on $\sigma^2$. We define Equation~\ref{eq:intro-rmarg} as the marginal $R^2$, since it's marginal of the effect of the design and any other kernel parameters.
The magnitude $\alpha_d^2$, however is not equal to the signal variance. %
To see this, consider the total variance of component $f_d(x)$,
\begin{equation}
  \begin{aligned}
  \Var(f_d(x)) = \alpha_d^2
  ={}& \underbrace{\E_{f_d}\bigl[\Var_x(f_d(x)\mid f_d)\bigr]}_{\text{signal variance}} \\
  &+ \underbrace{\Var_{f_d}\bigl(\E_x[f_d(x)\mid f_d]\bigr)}_{\text{variance of the level}} .
  \end{aligned}
  \label{eq:intro-decomp}
\end{equation}
The first term is the expected variance of the function across the observed inputs, while the second term captures the variance of the average height of function realisations  (See Appendix~\ref{app:calibration}).

The second term is precisely the point of departure between a correlated GP and a linear regression. If $f(x)$ were a linear term with mean-centred inputs, the level, $\E_x[f_d(x)\mid f_d]\bigr)$, is deterministically zero, such that marginal variance equals signal variance. 
In contrast, the average level of a realised GP acts as a random constant shift across all input observations. Because this shift is constant with respect to $x$, it does not contribute to the function's variation over the inputs (the signal variance), and the level itself is completely absorbed by the intercept $\mu_0$ in \eqref{eq:gp-gams}.

The size of the level variance depends on the length scale $\rho_d$ and on the input design. Denoting by $\bar{k}_d = \E_{x,x'}[k(x,x';\rho_d)]$ the average kernel correlation between two independent inputs $x$ and $x'$, the level variance reduces to $\alpha_d^2\,\bar{k}_d$, and the signal variance is therefore
\begin{equation}
  \begin{gathered}
  \E_{f_d}\bigl[\Var_x(f_d(x)\mid f_d)\bigr] = \alpha_d^2\, Q_d(\rho_d),
  \\
  Q_d(\rho_d) = 1 - \bar{k}_d .
  \end{gathered}
  \label{eq:intro-Q}
\end{equation}
For the Mat\'ern family, whose correlations are positive and increase with the length scale, $Q_d(\rho_d) \in (0,1)$ (unless all inputs coincide). %
As Figure~\ref{fig:variance-gap} illustrates, smoother realisations (larger $\rho_d$) carry a larger part of their variance in the level, so that their signal variance falls well below the magnitude $\alpha_d^2$.

\begin{figure*}[t]
  \centering
  \resizebox{1\textwidth}{!}{\input{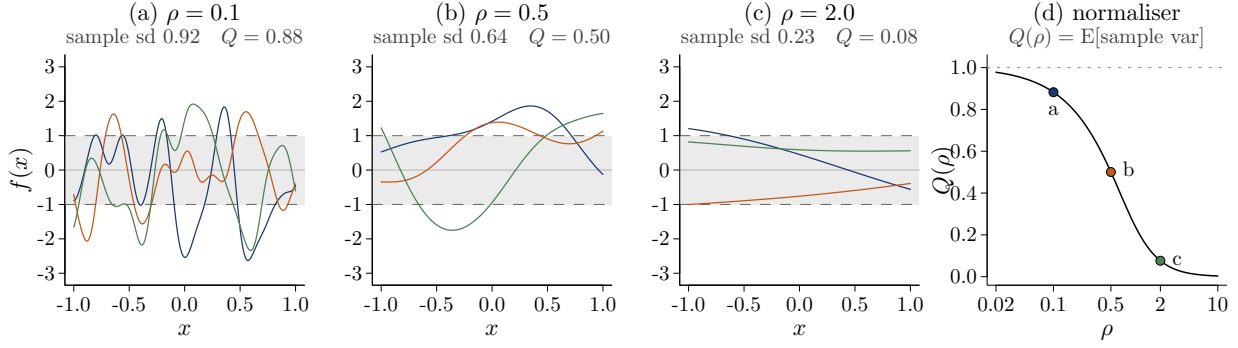}
    }
  \caption{GP Realisations and Normaliser $Q(\rho)$. Realisations of a zero-mean GP with an SE kernel and unit marginal variance on $x \in [-1,1]$. The shaded band is $\pm 1$ marginal standard deviation. Panels (a) to (c) hold the marginal variance fixed and vary the length scale, and report the sample standard deviation averaged over $4000$ realisations. As $\rho$ grows, the signal variance falls even though the marginal variance is unchanged. Panel (d) shows the resulting normaliser $Q(\rho)$ of \eqref{eq:intro-Q}.}
  \label{fig:variance-gap}
\end{figure*}

Hence, we define the \emph{realised} $R^2$ as the modelling target:
\begin{equation}
  \Rreal = \frac{\sum_{d=1}^{D}\alpha_d^2\,Q_d(\rho_d)}{\sum_{d=1}^{D}\alpha_d^2\,Q_d(\rho_d) + \sigma^2} ,
  \label{eq:intro-r2}
\end{equation}
which, whenever $Q_d(\rho_d) < 1$, implies $\Rreal < \Rmarg$.
Furthermore, because $Q_d$ depends on $\rho_d$, the implied prior on explained variance depends on the length scales, and so does the share of each component in the total signal variance.

This motivates replacing \eqref{eq:intro-model} with a normalised prior,
\begin{equation}
  f_d = \frac{\alpha_d}{\sqrt{Q_d(\rho_d)}}\, g_d .
  \label{eq:intro-scaled}
\end{equation}
Given $\rho_d$, this is a deterministic rescaling that leaves $f_d$ Gaussian and its correlation structure unchanged. By \eqref{eq:intro-Q}, it makes the signal variance of $f_d$ equal to $\alpha_d^2$ for every length scale, while the marginal variance becomes $\alpha_d^2/Q_d(\rho_d)$. Consequently, $\Rmarg$ and $\Rreal$ agree, and, as in the linear case, we can place a prior on a global $\Rsq$ and partition it across components with a Dirichlet distribution,
\begin{equation}
  \begin{gathered}
  \Rsq \sim \mathrm{Beta}\bigl(\mu_R \phi_R,\ (1-\mu_R)\phi_R\bigr),
  \qquad
  \alpha^2 = \sigma^2\frac{\Rsq}{1-\Rsq},
  \\
  \alpha_d^2 = \psi_d\, \alpha^2,
  \qquad
  \psi \sim \mathrm{Dirichlet}(a \one_D),
  \end{gathered}
  \label{eq:r2d2}
\end{equation}
where $\mu_R \in (0,1)$ and $\phi_R > 0$ parameterise the Beta prior through its mean and precision \citep{aguilar2023r2d2m2}. We refer to \eqref{eq:intro-scaled} together with \eqref{eq:r2d2} as the R2-GP prior. Its total signal variance is $\alpha^2$, and the Dirichlet weight $\psi_d$ is the share of component $d$, for any length scale value. Since the proportions in a Dirichlet distribution have strictly negative pairwise correlation, the GAM components intrinsically compete for their share of the total explained variance.

\subsection{The Hilbert Space Representation}
\label{sec:hsgp}
The Hilbert space (HS) approximation methods is based on the eigenvalue decomposition of the Gram matrix for a chosen kernel using its spectral density along with eigenvalues and eigenfunctions based on the GP inputs.
Following \citet{solin2020hilbert}, we use a compact cover $[-L_d, L_d]$ containing the observed values of covariate $d$. Based on the Dirichlet boundary conditions, we get closed form solutions for eigenvalues $\omega$ and eigenfunctions $\phi$ of the kernel such that
\begin{equation}
  \begin{gathered}
  \phi_{dm}(x) = \frac{1}{\sqrt{L_d}}\sin\bigl(\omega_{dm}(x + L_d)\bigr),
  \\
  \omega_{dm} = \frac{m\pi}{2L_d},
  \qquad m = 1,\dots,M,
  \end{gathered}
  \label{eq:basis}
\end{equation}
where $M$ is the number of basis functions. 
The matrix $\Phi_d \in \R^{N \times M}$ contains the eigenfunction values at the observed inputs, such that $[\Phi_d]_{im} = \phi_{dm}(x_{id})$, with $\phi_{dm} \in \R^N$ for its $m$-th column. We obtain the spectral density $S_{dm} = S(\omega_{dm}; \rho_d) > 0$ of our chosen kernel evaluated at the $m$-th frequency, collected in $S_d = (S_{d1},\dots,S_{dM})\T$. For example, the spectral density of the SE kernel $k(r;\rho) = \exp\{-r^2/(2\rho^2)\}$ is of the form $S(\omega;\rho) = \sqrt{2\pi}\,\rho\exp(-\rho^2\omega^2/2)$ obtained by solving the Fourier integral with respect to the inputs. The spectral forms of other Mat{\'e}rn kernels are provided in \citet{riutortmayol2023practical}.

\begin{lemma}[Finite-design covariance]
\label{lem:g}
Let $z_d \sim N(0, I_M)$ and $g_d = \Phi_d\diag(\sqrt{S_d})z_d \in \R^N$. Then,
conditionally on $X$ and $\rho_d$,
\begin{equation}
  g_d \sim N\bigl(0, \Kt_d\bigr),
  \qquad
  \Kt_d = \Phi_d \diag(S_d)\Phi_d\T,
  \label{eq:Ktilde}
\end{equation}
and $\Kt_d$ is symmetric positive semi-definite with $\rank \Kt_d \le \min(N,M)$.
\end{lemma}

The matrix $\Kt_d$ is the finite-basis approximation to the exact kernel matrix $K_d$ with $[K_d]_{ij} = k(x_{id}-x_{jd}; \rho_d)$. Under a sufficiently large boundary and increasing $M$ it converges to $K_d$
We use this convergence as established and refer to \citet{solin2020hilbert} for the underlying expansion. For practical guidance on choosing $L_d$ and $M$, we follow \citet{riutortmayol2023practical}. Since $k(0) = 1$, the diagonal of $\Kt_d$ is approximately one, which is the marginal normalisation that Section~\ref{sec:intro-gap} identified as insufficient for modelling the signal variance of the function.

\subsection{The Finite-sample Normaliser}
\label{sec:normalisation}

In HS form, we define the normaliser of component $d$ as the expected sample variance of the unit-magnitude component $g_d$ of Lemma~\ref{lem:g}, evaluated at the observed inputs,
\begin{equation}
  Q_d(\rho_d) \;:=\; \Ez\bigl[\vhat(g_d) \mid X, \rho_d\bigr],
  \label{eq:Qdef}
\end{equation}
where $\vhat(v) = (N-1)^{-1}\sum_{i=1}^{N}(v_i - \bar{v})^2$ is the sample variance of a vector $v \in \R^N$ and $\Ez$ is the expectation over the coefficients $z_d$. This is the finite-sample counterpart of \eqref{eq:intro-Q}, with the sample variance over the observed inputs in place of the variance operator $\Var_x$ and the HSGP covariance $\Kt_d$ in place of the exact kernel. Evaluating \eqref{eq:Qdef} through the $N \times N$ matrix $\Kt_d$ would cost $O(N^2M)$ operations each time $\rho_d$ changes. However, because the coefficients in $z_d$ are independent with unit variance, the cross terms between different basis functions have expectation zero, and the normaliser simplifies to the sum of the sample variances of the basis functions, weighted by the spectral density (Lemma~\ref{thm:Q} in Appendix~\ref{app:proofs-hsgp}),
\begin{equation}
  Q_d(\rho_d) = \sum_{m=1}^{M} S_{dm}\,\vhat(\phi_{dm}).
  \label{eq:Qcode}
\end{equation}

This identity allows the R2-GP prior to scale well with $N$. The sample variances $\vhat(\phi_{dm})$ of the basis columns depend on the design alone and not on any parameter, so they are precomputed once at cost $O(NM)$. Inside the sampler, the normaliser then reduces to an inner product of length $M$ between the spectral density $S_d$ and these precomputed variances, which requires only $O(M)$ work each time the length scale $\rho_d$ changes.\footnote{The normaliser is strictly positive whenever $M \ge 2$ and the observed inputs are not all equal (Lemma~\ref{thm:Q}), so that the division by $\sqrt{Q_d(\rho_d)}$ in Section~\ref{sec:normalised-prior} is well defined.}

Expanding the sample variance in \eqref{eq:Qdef} gives the HSGP counterpart of the decomposition \eqref{eq:intro-decomp},
\begin{equation}
  Q_d(\rho_d)
  = \frac{1}{N-1}\sum_{i=1}^{N}[\Kt_d]_{ii}
    \;-\; \frac{N}{N-1}\Var_z(\bar{g}_d),
  \label{eq:deficit}
\end{equation}
where $\bar{g}_d$ is the sample mean of the realisation $g_d$. The first term aggregates the marginal variances, each approximately one, and the second is the prior variance of the realisation's level, which the intercept absorbs. Together they equal the average marginal variance minus the average covariance between two distinct observed inputs, the sample analogue of $1 - \bar{k}_d$ in \eqref{eq:intro-Q}. %

\subsection{The Normalised Additive Prior in HS Form}
\label{sec:normalised-prior}

We construct the normalised HSGPs, using $\alpha_d^2$ from \eqref{eq:r2d2} as follows
\begin{equation}
  \begin{gathered}
  f_d
  = \frac{\alpha_d}{\sqrt{Q_d(\rho_d)}}\, g_d
  = \Phi_d\bigl(c_d \odot z_d\bigr),
  \\
  c_d = \alpha_d\sqrt{\frac{S_d}{Q_d(\rho_d)}},
  \qquad
  f = \mu_0 + \sum_{d=1}^{D} f_d ,
  \end{gathered}
  \label{eq:fd}
\end{equation}
where $\odot$ is the elementwise product. The square root and division are
elementwise. Together with \eqref{eq:r2d2}, \eqref{eq:fd} is the HSGP
form of the R2-GP prior.

As for the exact GP in Section~\ref{sec:intro-gap}, the signal variance of each component is then its allocated variance,
\begin{equation}
  \Ez\bigl[\vhat(f_d) \mid X, \rho_d, \alpha_d\bigr] = \alpha_d^2 ,
  \qquad d = 1,\dots,D,
  \label{eq:calib-comp}
\end{equation}
and, for independent $z_1,\dots,z_D$, the signal variances add up to the total budget,
\begin{equation}
  \Ez\Bigl[\vhat\Bigl(\sum_{d=1}^{D} f_d\Bigr) \Big| X, \rho, \alpha\Bigr]
  = \sum_{d=1}^{D}\alpha_d^2
  = \alpha^2 .
  \label{eq:calib-total}
\end{equation}
Three comments are in order that highlight the usefulness of the $\mathrm{R2}$ prior for GPs.
Firstly, by writing $f_d = \Phi_d\beta_d^{\mathrm{eff}}$ with $\beta_d^{\mathrm{eff}} = c_d \odot z_d$, the proposed prior is equivalent to independent centred priors on the basis weights
\begin{equation}
  \beta^{\mathrm{eff}}_{dm} \sim
  N\Bigl(0,\ \alpha_d^2 S_{dm} / Q_d(\rho_d)\Bigr),
  \qquad m = 1,\dots,M .
  \label{eq:coefprior}
\end{equation}
Computationally, the $z_d$ remains standard normal and the scale is applied through a deterministic transformation, and thus, equivalent to the non-centred parameterisation \citep{papaspiliopoulos2007general} which is well known to keep the posterior geometry well conditioned when a signal variance is small. %

Second, because the deterministic rescaling in \eqref{eq:intro-scaled} applies a single global scalar $1/\sqrt{Q_d(\rho_d)}$ to all basis functions of component $d$ equally, it preserves the relative weighting of the spectral densities $S_{dm}$ that govern the smoothness of the function. Furthermore, because this scalar multiplier is independent of the basis index $m$, it does not interact with the truncation of the series. Consequently, the normalised HSGP inherits the same convergence guarantees as the standard HSGP: as the number of basis functions $M$ grows, the approximate prior converges to the exact normalised GP prior.

\begin{figure*}[t]
  \centering
  \resizebox{0.9\textwidth}{!}{\input{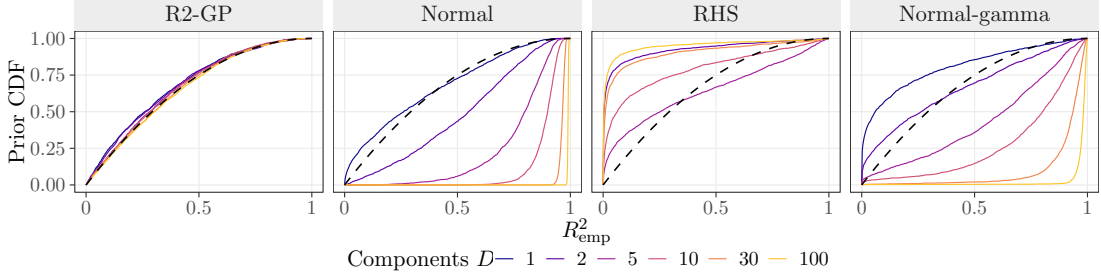}
    }
  \caption{Implied Prior on $\Remp$. Empirical CDF of $\Remp = \vhat(f)/\bigl(\vhat(f) + \sigma^2\bigr)$ for additive HSGP GAMs under four priors on $\alpha_d$, as the number of components $D$ grows. CDFs are generated from $2000$ prior draws: $N = 200$, $M = 100$, $\sigma = 1$ and $\rho_d \sim \mathrm{InvGamma}(5, 4)$. These figures are insensitive to the prior $\rho_d$. RHS and normal-gamma are specified as in Section~\ref{sec:simulation}, with $p_0 = 5$. The dashed curve is the target $\mathrm{Beta}(1,2)$. Only R2-GP is stable in $D$. %
  }
  \label{fig:prior-density}
\end{figure*}

Third,\label{sec:prior-predictive} the prior on the total signal variance does not change as components are added. By \eqref{eq:calib-total} and \eqref{eq:r2d2}, it equals $\sigma^2\Rsq/(1-\Rsq)$ for every $D$, so that its prior is the one placed on $\Rsq$. If instead each $\alpha_d$ has a prior second moment $s^2$ that is free of $D$, as under independent priors, it has prior mean $Ds^2$. Figure~\ref{fig:prior-density} shows the consequence for prior draws for the empirical proportion of explained variance, $\Remp = \vhat(f)/\bigl(\vhat(f) + \sigma^2\bigr)$ under different priors. Under the R2-GP prior with $\Rsq \sim \mathrm{Beta}(1,2)$ and $a = 1$ (panel R2-GP), its implied prior stays close to the target for every $D$ between $1$ and $100$. Under independent half-normal priors with unit scale, the median rises from $0.26$ at $D = 1$ to $0.99$ at $D = 100$. As competing shrinkage priors, which are further introduced in Section~\ref{sec:simulation}, the regularised horseshoe \citep{piironen2017sparsity} and normal-gamma \citep{griffin2010normalgamma} may shrink each $\alpha_d$ but leave the prior on their total to shrink or grow with $D$.

Notice that calibration holds for $\vhat(f)$ but not for the empirical proportion of explained variance, which is a nonlinear function of it. The map $t \mapsto t/(t+\sigma^2)$ is strictly concave on $t > 0$, so Jensen's inequality gives
\begin{equation}
  \E\bigl[\Remp \mid X, \sigma, \Rsq\bigr]
  \;\le\; \frac{\alpha^2}{\alpha^2+\sigma^2} \;=\; \Rsq ,
  \label{eq:jensen}
\end{equation}
with strict inequality whenever $\vhat(f)$ is non-degenerate. This accounts for the slight shift to the left for small $D$ in Figure~\ref{fig:prior-density} for the R2-GP prior. This is not a deficiency of the proposed prior, since no choice of $Q_d$ calibrates both $\E[\vhat(f)]$ and $\E[h(\vhat(f))]$ for a nonlinear $h$.

Finally, to ensure the component's prior expected marginal variance remains finite, the prior on the length scale $\rho_d$ must be sufficiently light-tailed. As $\rho_d \to \infty$, the normaliser for the exact kernel decays such that $1/Q_d(\rho_d) \asymp \rho_d^{\gamma}$, where $\gamma = 1$ for the exponential kernel and $\gamma = 2$ for smoother Mat\'ern variants \citep{rasmussen2006gaussian}  . Therefore, the exact marginal variance $\alpha_d^2/Q_d(\rho_d)$ has a finite prior expectation if and only if $\E[\rho_d^{\gamma}] < \infty$. While the HSGP approximation inherently bounds this variance for finite $M$, preserving the properties of the exact-kernel limit rules out heavy-tailed priors, such as an $\mathrm{Inverse\text{-}Gamma}(0.1, 0.1)$ distribution.

\section{Simulation study}
\label{sec:simulation}

We assess finite-sample predictive and parameter recovery performance of the R2-GP prior, that is, \eqref{eq:r2d2} together with \eqref{eq:fd}, as the number of GAM components $D$ increases. The competitors are priors on the component scales that carry popular shrinkage approaches over to \eqref{eq:gp-gams}. These have not yet been considered for HSGP GAMs to our knowledge, but pose straightforward extensions based off our methodology.

\begin{figure*}[t]
  \centering
  \resizebox{1\textwidth}{!}{\input{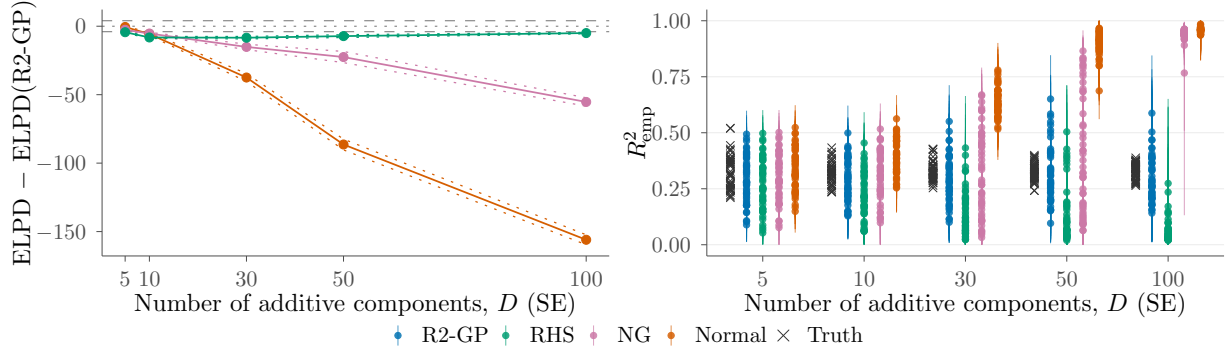}
    }
  \caption{Dense DGP, SE Kernel, $\rho_0 = 0.3$. Left: mean $\Delta\mathrm{elpd}$ over $50$ simulations (solid) with $\pm 1$ Monte Carlo standard error (dotted). Dashed lines mark $\pm 4$ and negative values favour R2-GP. Right: posterior means (points) and central $90\%$ intervals (bars) of $\Remp$ for each simulation. Crosses mark the realised $\Rsq_{\mathrm{true}}$.}
  \label{fig:sim-r2}
\end{figure*}

\subsection{Setup}
\label{sec:sim-setup}

\paragraph{Data Generation.}
We generate data from \eqref{eq:gp-gams} with $\mu_0 = 0$ and $\sigma = 1$. Each $f_d$ is a draw from the normalised prior \eqref{eq:fd} with a common length scale $\rho_d = \rho_0$ and signal variance
\begin{equation}
  \alpha_d^2 = \frac{\sigma^2}{D}\,\frac{\Rsq_0}{1-\Rsq_0},
  \qquad \Rsq_0 = \frac{1}{3},
  \label{eq:sim-alpha}
\end{equation}
for $d = 1,\dots,D$. By \eqref{eq:calib-total}, the total signal variance is $\sigma^2\Rsq_0/(1-\Rsq_0) = 1/2$ for every $D$, while the signal variance of a single component decays at rate $D^{-1}$. We call \eqref{eq:sim-alpha} the dense data-generating process (DGP). We also consider a sparse DGP in Appendix~\ref{app:sim-sparse}, in which the first five components share the same total signal variance and $\alpha_d = 0$ for $d > 5$. We vary $D \in \{5,10,30,50,100\}$, $\rho_0 \in \{0.3,0.5,0.7\}$ and the kernel choice that includes SE, Mat\'ern $3/2$ and Mat\'ern $5/2$. For each condition, we simulate $50$ replications. A data set has $N = 120$ training and $500$ test observations. Inputs are iid uniform on $[-1,1]$. Following \citet{riutortmayol2023practical}, we set $L_d = \max(1.5,\,4.5\rho_0)$ and $M = \lceil 1.2\,\kappa L_d/\rho_0\rceil$ with $\kappa = 1.75$, $3.42$ and $2.65$ for the three kernels, so that $10 \le M \le 21$. The data are generated with $4M$ basis functions, so that the $f$ is generated close as a close representation to the desired GP. Since \eqref{eq:calib-total} holds in expectation over $z$, the realised value $\Rsq_{\mathrm{true}}$, that is, $\Remp$ at the generating $f$ and $\sigma$, varies across data sets around the nominal $\Rsq_0$. %

\paragraph{Priors.}
All models use the normalised components \eqref{eq:fd} and differ only in the prior on $\alpha_1,\dots,\alpha_D$. R2-GP is \eqref{eq:r2d2} with $\mu_R = 1/3$, $\phi_R = 3$ and $a = 1$, so that $\Rsq \sim \mathrm{Beta}(1,2)$, as in Section~\ref{sec:prior-predictive}. We compare it with three alternatives, where a superscript $+$ denotes truncation to $(0,\infty)$.
(i) The \emph{normal} prior, $\alpha_d \sim N^{+}(0,1)$ independently, is the analogue of ridge regression and represents independent priors on the component scales.
(ii) The regularised horseshoe \citep[RHS;][]{piironen2017sparsity}, $\alpha_d \mid \lambda_d,\tau,\zeta \sim N^{+}(0,\tau^2\tilde\lambda_d^2)$ with $\tilde\lambda_d^2 = \zeta^2\lambda_d^2/(\zeta^2+\tau^2\lambda_d^2)$, $\lambda_d \sim t^{+}_{3}(0,1)$, $\tau \sim C^{+}(0,\tau_0)$ and $\zeta^2 \sim \mathrm{InvGamma}(2,8)$, represents sparsity-favouring priors. We set $\tau_0 = p_0(D-p_0)^{-1}\sigma N^{-1/2}$ as in \citet{piironen2017sparsity}, with a prior guess of $p_0 = \min(5,D-1)$ active components. While $p_0 = 5$ is miscalibrated for the dense DGP, $\tau$ is nevertheless learned from the data, and $p_0$ is well calibrated for the sparse DGP.
(iii) The normal-gamma prior \citep[NG;][]{griffin2010normalgamma}, $\alpha_d \mid \xi_d \sim N^{+}(0,\sigma^2\xi_d)$ with $\xi_d \sim \mathrm{Gamma}(\vartheta,\vartheta/\eta)$ in shape-rate form, $\vartheta \sim \mathrm{Exp}(1)$ and $\eta \sim \mathrm{InvGamma}(2,1)$, represents global-local shrinkage priors that are popular in other domains.
The remaining priors are common to all models, namely $\rho_d \sim N^{+}(\rho_0,0.1^2)$, $\sigma \sim N^{+}(0,1)$ and $\mu_0 \sim N(0,2^2)$. %
The comparisons presented here isolate the effect of the prior on $\alpha_d$. We show in Appendix~\ref{app:sim-unnorm} that the qualitative behaviour of the competing priors is the same also without the normalisation.

\paragraph{Expected Behaviour.}
The total signal variance is independent of $D$ under R2-GP, whereas its mean is $D$ under normal, and $D\sigma^2$ under NG, respectively. Since the generating value is $1/2$ throughout, expect  normal and NG to overestimate the signal variance and to underestimate $\sigma$ by a margin that grows with $D$, in line with Figure~\ref{fig:prior-density}. RHS is a sparsity favouring prior which for $p_0 = 5$ has $\tau_0 = O(D^{-1})$, hence, decays faster than the generating $\alpha_d = O(D^{-1/2})$ which is expected to lead to excessive shrinkage for large $D$.

\paragraph{Evaluation.}
Predictive performance is measured by the log predictive density summed over the test observations, which estimates the expected log predictive density \citep[elpd;][]{vehtari2017practical}, relative to R2-GP,
\begin{equation}
  \Delta\mathrm{elpd}_m^{(s)}
  = \mathrm{elpd}_m^{(s)} - \mathrm{elpd}_{\text{R2-GP}}^{(s)},
  \label{eq:sim-delpd}
\end{equation}
for prior $m$ and replication $s$, so that negative values favour R2-GP. We report its mean over replications and, following \citet{sivula2025uncertainty} and \citet{mclatchie2025advances}, regard $|\Delta\mathrm{elpd}| > 4$ as a meaningful difference. For $\theta \in \{\alpha,\rho,\sigma\}$ we report the coverage of the central $90\%$ posterior intervals (Tables~\ref{tab:coverage-dense-se}, \ref{tab:coverage-dense-m32} and \ref{tab:coverage-dense-m52}) and the relative root-mean-squared error, that is relative to the RMSE of the R2-GP (RMSE; Tables~\ref{tab:rmse-dense-se}, \ref{tab:rmse-dense-m32} and \ref{tab:rmse-dense-m52})

Values above one indicate worse recovery than under R2-GP. Lastly, we compare the posterior of $\Remp$ with $\Rsq_{\mathrm{true}}$. All models are implemented in Stan \citep{carpenter2017stan} and fitted with four chains of $2000$ iterations, with the half as warm-up at a target acceptance rate of $0.95$. Figure~\ref{fig:sim-r2} shows the predictive performance and the recovery of $\Rsq_{\mathrm{true}}$ for the SE kernel with $\rho_0 = 0.3$. We show figures for all other conditions in Appendix~\ref{app:simulation}. See  Tables~\ref{tab:sampling-dense-se}, \ref{tab:sampling-dense-m32} and \ref{tab:sampling-dense-m52} for sampling diagnostics.

\subsection{Results}
\label{sec:sim-results}

\paragraph{Predictive Performance.}
R2-GP attains the highest mean elpd in each of the $45$ conditions, against every competitor. At $D = 5$ the priors are close, however normal and NG deteriorate as $D$ grows. This is as a consequence of the lack of joint shrinkage of the GAM components, which leads to overestimation of $\Rsq_{\mathrm{true}}$. RHS stays closest to R2-GP at higher $D$, with mean $\Delta\mathrm{elpd} \in [-13.5, -4.0]$ for all $D \ge 10$. As expected, with $p_0 = 5$ RHS overshrinks as $D$ increases, which can be seen as underestimation of $\Rsq_{\mathrm{true}}$ in Figure~\ref{fig:sim-r2}. Due to the small true variance shares as $D$ increases, this results in less of a penalty in predictive performance than the overestimation of $\Rsq_{\mathrm{true}}$ seen for normal and NG.

\paragraph{Recovery.}
Recovery mirrors the predictive results for the dense DGP. R2-GP has the lowest or close to the lowest RMSE for $(\alpha,\sigma)$ in all conditions. Interval coverage is close to or above the nominal level, with its advantage widening with $D$. (Tables~\ref{tab:coverage-dense-se} and \ref{tab:rmse-dense-se} for the SE kernel, Appendix~\ref{app:simulation} for the Mat\'ern kernels).%

\paragraph{Sampling Efficiency.}
R2-GP also improves the sampling efficiency markedly. On average over replications, its largest $\widehat{R}$ is at most $1.01$ and its smallest bulk effective sample size is at least $448$ in every condition. It has the fewest divergent transitions in all conditions and the shortest sampling time in all but one. At $D = 100$, by contrast, the normal prior and NG fail to converge. The results of the normal and NG priors at large $D$ therefore reflect the prior and the sampler jointly. %
The RHS prior converges, with $\widehat{R} \le 1.03$, but needs up to $1.6$ times the sampling time.

\paragraph{Sparse Allocation.}
Under the sparse DGP (Appendix~\ref{app:sim-sparse}), RHS with the calibration $p_0 = 5$ attains a higher mean elpd than R2-GP for $D \ge 30$, however by small margins. The normal prior underperforms as in the dense DGP, and NG does so at $D = 100$. %
Relative to RHS, R2-GP's RMSE is slightly lower for the non-zero $\alpha_d$ in all conditions, however higher for the true zero $\alpha_d$. This is as expected, since for $a = 1$ the allocation $\psi$ is uniform on the simplex, whereas concentration near zero requires a small $a$ \citep{zhang2022r2d2,aguilar2023r2d2m2,kohns2025arr2}.

\paragraph{Summary.}
For dense signals R2-GP performs at least as well as every competing prior in terms of prediction and recovers component scales most accurately. In particular, our method recovers the component scales most accurately in all conditions, with a margin over the normal and NG priors that grows with $D$. Sampling wise, our prior is the most reliable prior to sample. The patterns are almost identical across the three kernels and the three length scales.

\section{Real-world Case Study} \label{sec:case_study}

\paragraph{Data.}
In addition to our simulation scenarios, we demonstrate our proposed methods in analysing the BloodBrain data \citep{mente2005recursive} from the R package \textit{caret} \citep{kuhn2008caret}. The dataset describes $208$ chemical compounds using $134$ molecular descriptors based on the compounds structure. The response $y$ is the logarithm of the ratio of the compound's concentration in the brain to the blood, a standard procedure to measure how readily a chemical compound crosses the blood-brain barrier (with \textit{logBBB} as response). For simplicity, we consider only continuous covariates. The data has $N=208$ samples and $D = 101$ covariates. 

\paragraph{Model Fit.}
We estimate the same set of models as in Section~\ref{sec:simulation}, except for the RHS prior for which $p_0 = 10$. 
We consider the SE and  Mat{\'e}rn 3/2 kernels.
We set boundary conditions with $c=2$ and compute the number of basis functions accordingly.
To compare predictive performance, the data is split into $80/20$ train/test ratio with 10 random replications. %
Each fit uses 4 chains with 1000 warm-up and 1000 sample draws. %

\begin{figure}[t]
  \centering
  {\input{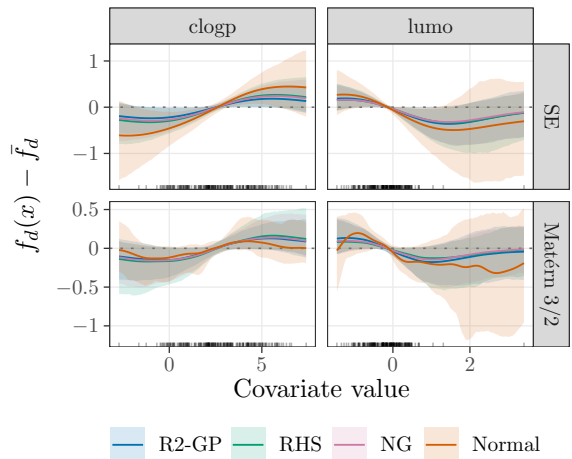}
    }
  \caption{Covariate Effects, BloodBrain Data. Top 2 covariates are shown based on their contribution to explained variance.}
  \label{fig:case-study-2effects}
\end{figure}

\paragraph{Results.}
The covariate \textit{clogp} describing lipophilicity is a biologically relevant covariate for modelling \textit{logBBB} \citep{mente2005recursive} in the BloodBrain data. Additionally, we also investigate here \textit{lumo} describing the energy component. 
The effects of these two covariates on response \textit{logBBB} for all fitted models are shown in Figure \ref{fig:case-study-2effects}. Across both kernels and all models, \textit{logBBB} increases with an increase in \textit{clogp}, indicating lipophilic compounds cross the blood-brain-barrier readily. For \textit{lumo}, we see a decreasing effect on \textit{logBBB}. The shrinkage priors agree while maintaining tight posterior credible intervals while the Normal model remains unstable. %
See Appendix~\ref{app:add-case-study} for the effect of more covariates.  %
Our proposed R2-GP prior provides superior predictive performance as well as sampling speed and diagnostics (Appendix Table~\ref{tab:casestudy-sample-diag}) compared to the Normal model. RHS and NG provide comparable predictive performance R2-GPs, but worse MCMC convergence, the latter also reflected in the simulation study. Overall, using R2-GP is computationally and inferentially more robust for this data while providing the benefit of a prior based off of first principles.

\section{Discussion}
We introduce the R2-GP prior for GAM GPs under which the magnitude of a component equals its signal variance, the expected variation across the inputs. We model magnitudes based off of an interpretable prior on the model's proportion of explained variance and a Dirichlet partition of that variance across components. Thus, components compete for explained variance, preventing variance explosion with many components. In HS form, the normaliser costs $O(M)$ per gradient evaluation, so the approach also scales well for large sample sizes. Our simulations and real-data case study demonstrate superior predictive performance of R2-GP over independent normal priors. Although competitive in predictive performance with other shrinkage priors, R2-GP outperforms in terms of sampling efficiency, especially when many components are active. Future works should include richer kernel structures, such as interactions between covariates and as well as applications in other GP extensions.

\section*{AI Use Statement}

In this work, we used generative AI tools for code prototyping of parts of the code (Claude Opus 5 High from Anthropic). Further, we used the ChatGPT 5.6 model from OpenAI for minor language edits, aiming to enhance readability. LLM-generated code and text was verified, modified and tested for correctness by the authors. We take responsibility for the final content of this work, including text, claims, or artifacts produced with the aid of generative AI.

\bibliography{refs}

\clearpage
\appendix
\supplementsizes

\begin{center}
  {\Large\bfseries Supplementary Material\par}
\end{center}
\vspace{1em}

\paragraph{Overview.}
This supplement has four parts. Appendix~\ref{app:proofs} gives the derivations and proofs for Section~\ref{sec:r2hsgp}. Appendix~\ref{app:simulation} gives the full results of the simulation study of Section~\ref{sec:simulation}. Appendix~\ref{app:add-case-study} gives further results for the case study of Section~\ref{sec:case_study}. Appendix~\ref{app:compute} describes the computing infrastructure. Table~\ref{tab:app-overview} lists the content of each part and the section of the main text that it supports.

\begin{table}[!htbp]
\apptable
\caption{Overview of the Supplementary Material}\label{tab:app-overview}
\begin{tabular}{@{}l p{0.56\textwidth} l l l@{}}
\toprule
Appendix & Content & Figures & Tables & Section \\
\midrule
\ref{app:calibration} & Variance decomposition \eqref{eq:intro-decomp} and normaliser \eqref{eq:intro-Q} for the exact GP & & & \ref{sec:intro-gap} \\
\ref{app:proofs-hsgp} & Proofs of Lemmas~\ref{lem:g} and \ref{thm:Q}, of the decomposition \eqref{eq:deficit} and of the calibration results \eqref{eq:calib-comp} and \eqref{eq:calib-total} & & & \ref{sec:hsgp}--\ref{sec:normalised-prior} \\
\addlinespace
\ref{app:sim-dense} & Dense DGP: $\Delta\mathrm{elpd}$ and posterior of $\Remp$ (figures), interval coverage, relative RMSE and sampling diagnostics (tables), for all kernels and length scales & \ref{fig:dense-elpd-se}--\ref{fig:dense-r2-m52} & \ref{tab:coverage-dense-se}--\ref{tab:sampling-dense-m52} & \ref{sec:sim-results} \\
\ref{app:sim-sparse} & Sparse DGP: the same figures and tables & \ref{fig:sparse-elpd-se}--\ref{fig:sparse-r2-m52} & \ref{tab:coverage-sparse-se}--\ref{tab:sampling-sparse-m52} & \ref{sec:sim-results} \\
\ref{app:sim-unnorm} & Competing priors without the normalisation: $\Delta\mathrm{elpd}$ for both DGPs & \ref{fig:unnorm-elpd-dense-se}--\ref{fig:unnorm-elpd-sparse-m52} & & \ref{sec:sim-setup} \\
\addlinespace
\ref{app:add-case-study} & BloodBrain data: effects of the $12$ covariates with the largest share of explained variance, predictive performance and sampling diagnostics & \ref{fig:case-study-12effects-se}--\ref{fig:case-study-12effects-m32} & \ref{tab:casestudy-sample-diag} & \ref{sec:case_study} \\
\addlinespace
\ref{app:compute} & Hardware, software versions and run times & & & \ref{sec:simulation}, \ref{sec:case_study} \\
\bottomrule
\end{tabular}
\end{table}

\section{Proofs}
\label{app:proofs}

\subsection{Derivations for Section~\ref{sec:intro-gap}}
\label{app:calibration}

Throughout, $\E_x$ and $\Var_x$ are the expectation and variance operators with respect to random variable $x$, and $x$ and $x'$ are independent inputs.

\paragraph{Decomposition \eqref{eq:intro-decomp}.}
Apply the law of total variance to $f_d(x)$ under the joint distribution of $x$ and $f_d$. Conditioning on $x$ gives
\[
  \Var\bigl(f_d(x)\bigr)
  = \E_x\bigl[\Var_{f_d}(f_d(x))\bigr] + \Var_x\bigl(\E_{f_d}[f_d(x)]\bigr)
  = \alpha_d^2 ,
\]
since under \eqref{eq:intro-model} the prior mean is zero and $\Var_{f_d}(f_d(x)) = \alpha_d^2\,k(x,x;\rho_d) = \alpha_d^2$ at every input. Conditioning on $f_d$ instead gives $\Var(f_d(x)) = \E_{f_d}[\Var_x(f_d(x)\mid f_d)] + \Var_{f_d}(\E_x[f_d(x)\mid f_d])$, and the two expressions together are \eqref{eq:intro-decomp}.

\paragraph{Linear Term.}
The same two steps hold for any zero-mean prior on a function $f$ of $x$, with $\E_x[\Var_f(f(x))]$ on the left-hand side. For $f(x) = x\beta$ with centred inputs, $\E_x[x] = 0$, the level is $\E_x[x]\,\beta = 0$ for every $\beta$. Its variance is therefore zero, and the marginal variance, averaged over the inputs, equals the signal variance, $\Var(\beta)\,\E_x[x^2]$.

\paragraph{Variance of the Level and $Q_d$.}
Under \eqref{eq:intro-model}, the level has mean zero, and $x$, $x'$ and $f_d$ are independent, so that
\[
  \Var_{f_d}\bigl(\E_x[f_d(x)\mid f_d]\bigr)
  = \E_{f_d}\bigl[\E_x[f_d(x)\mid f_d]\;\E_{x'}[f_d(x')\mid f_d]\bigr]
  = \E_{x,x'}\bigl[\E_{f_d}[f_d(x)f_d(x')]\bigr]
  = \alpha_d^2\,\E_{x,x'}\bigl[k(x,x';\rho_d)\bigr] ,
\]
which is $\alpha_d^2\,\bar{k}_d$ and, with \eqref{eq:intro-decomp}, gives \eqref{eq:intro-Q}. A Mat\'ern kernel has the form $k(x,x';\rho_d) = \kappa(|x-x'|/\rho_d)$, where $\kappa$ is positive and strictly decreasing with $\kappa(0) = 1$. Hence $0 < k(x,x';\rho_d) < 1$ for $x \neq x'$, and $k(x,x';\rho_d)$ increases with $\rho_d$. Unless all inputs coincide, it follows that $\bar{k}_d \in (0,1)$ increases with $\rho_d$, so that $Q_d(\rho_d) \in (0,1)$ decreases with $\rho_d$. For a general stationary kernel, $\alpha_d^2\bar{k}_d$ is still a variance that is at most $\alpha_d^2$, so that $Q_d \in [0,1]$ without any assumption on the sign of the correlations.

\paragraph{Finite-sample Counterpart.}
Section~\ref{sec:normalisation} replaces $\Var_x$ by the sample variance $\vhat$ over the observed inputs $x_{1d},\dots,x_{Nd}$. The proof of \eqref{eq:Qsplit} holds for any covariance matrix. For the exact kernel matrix $K_d$ in place of $\Kt_d$, with $[K_d]_{ii} = 1$, it gives
\[
  \E_{f_d}\bigl[\vhat(f_d)\bigr]
  = \alpha_d^2\Bigl\{1 - \frac{1}{N(N-1)}\sum_{i \neq j}k(x_{id},x_{jd};\rho_d)\Bigr\}
\]
under \eqref{eq:intro-model}, so that the expectation over two independent inputs in $\bar{k}_d$ becomes the average over pairs of distinct observed inputs.

\subsection{Proofs for Sections~\ref{sec:hsgp} to \ref{sec:normalised-prior}}
\label{app:proofs-hsgp}

Throughout, $\E_z$ denotes expectation over the coefficient vectors $z_1,\dots,z_D$ only, holding $X$, $\rho$ and $\alpha$ fixed. Two elementary identities are used repeatedly and are stated here once. Let $P_N = I_N - N^{-1}\one_N\one_N\T$ be the centring matrix, which is symmetric and idempotent with $P_N\one_N = 0$ and $\tr P_N = N-1$. Subtracting the sample mean is an orthogonal projection, so for any $v \in \R^N$
\begin{equation}
  \vhat(v) = \frac{1}{N-1}\lVert P_N v\rVert^2 = \frac{1}{N-1}\,v\T P_N v ,
  \label{eq:varquad}
\end{equation}
and for a random vector $u$ with mean zero and covariance $\Sigma$ and a fixed matrix $A$, writing the scalar $u\T Au$ as $\tr(Auu\T)$ and using linearity of the trace gives
\begin{equation}
  \E\bigl[u\T A u\bigr] = \tr(A\Sigma).
  \label{eq:quadexp}
\end{equation}

\begin{proof}[Proof of Lemma~\ref{lem:g}]
Write $g_d = A z_d$ with the fixed matrix $A = \Phi_d\diag(\sqrt{S_d})$. A linear image of a Gaussian vector is Gaussian, with $\E_z[g_d] = A\E_z[z_d] = 0$ and
\[
  \E_z[g_dg_d\T] = A\,\E_z[z_dz_d\T]\,A\T = AA\T
  = \Phi_d\diag(\sqrt{S_d})\diag(\sqrt{S_d})\Phi_d\T
  = \Phi_d\diag(S_d)\Phi_d\T,
\]
using $\diag(\sqrt{S_d})^2 = \diag(S_d)$, valid since the matrices are diagonal and $S_{dm} > 0$. Symmetry is immediate, and $v\T\Kt_d v = \lVert\diag(\sqrt{S_d})\Phi_d\T v\rVert^2 \ge 0$ gives positive semi-definiteness. The rank bound follows from $\rank(AA\T) = \rank(A) \le \min(N,M)$.
\end{proof}

\begin{lemma}[Equivalent expressions for the normaliser]
\label{thm:Q}
Let $g_d$ and $\Kt_d$ be as in Lemma~\ref{lem:g}. Then the normaliser \eqref{eq:Qdef} satisfies
\begin{align}
  Q_d(\rho_d)
  &= \frac{1}{N-1}\tr\bigl(P_N \Kt_d\bigr)
  \label{eq:Qtrace}\\
  &= \sum_{m=1}^{M} S_{dm}\,\vhat(\phi_{dm})
  \tag{\ref{eq:Qcode}}\\
  &= \frac{1}{N-1}\Bigl\{\tr(\Kt_d) - \tfrac{1}{N}\one_N\T \Kt_d \one_N\Bigr\}.
  \label{eq:Qsplit}
\end{align}
Moreover, $Q_d(\rho_d) > 0$ whenever $M \ge 2$ and the observed inputs of covariate $d$ lie in $(-L_d, L_d)$ and are not all equal.
\end{lemma}

\begin{proof}
By \eqref{eq:varquad}, $\vhat(g_d) = (N-1)^{-1}g_d\T P_N g_d$. Conditionally on $X$ and $\rho_d$ the matrix $P_N$ is fixed and $g_d$ has mean zero with covariance $\Kt_d$ by Lemma~\ref{lem:g}, so \eqref{eq:quadexp} gives \eqref{eq:Qtrace}.

For \eqref{eq:Qcode}, substitute $\Kt_d = \Phi_d\diag(S_d)\Phi_d\T$ and use cyclicity together with the fact that $\tr(\diag(S)B) = \sum_m S_m B_{mm}$ for any square $B$:
\begin{align*}
  \tr\bigl(P_N\Phi_d\diag(S_d)\Phi_d\T\bigr)
  &= \tr\bigl(\diag(S_d)\,\Phi_d\T P_N\Phi_d\bigr)\\
  &= \sum_{m=1}^{M}S_{dm}\bigl[\Phi_d\T P_N \Phi_d\bigr]_{mm}
   = \sum_{m=1}^{M}S_{dm}\,\phi_{dm}\T P_N\phi_{dm}.
\end{align*}
Dividing by $N-1$ and applying \eqref{eq:varquad} to each column turns $\phi_{dm}\T P_N \phi_{dm}/(N-1)$ into $\vhat(\phi_{dm})$. The off-diagonal entries of $\Phi_d\T P_N\Phi_d$, which are the sample covariances between different basis columns, do not enter, because the coefficients in $z_d$ are uncorrelated.

For \eqref{eq:Qsplit}, expand $P_N = I_N - N^{-1}\one_N\one_N\T$ inside the trace and use $\tr(\one_N\one_N\T\Kt_d) = \one_N\T\Kt_d\one_N$.

For strict positivity, $S_{dm} > 0$ and $\vhat(\phi_{dm}) \ge 0$, so that \eqref{eq:Qcode} is zero only if every basis column is constant across the observed inputs. Let $u_i = (x_{id} + L_d)/(2L_d)$, which lies in $(0,1)$ by assumption, so that $\phi_{dm}(x_{id}) = L_d^{-1/2}\sin(m\pi u_i)$ by \eqref{eq:basis}. The first column is constant only if every $u_i$ equals $u^*$ or $1 - u^*$ for some $u^* \in (0, 1/2]$. If both values occur and $u^* < 1/2$, the second column takes the two distinct values $\pm L_d^{-1/2}\sin(2\pi u^*)$ and is not constant. Otherwise all inputs are equal. The condition $M \ge 2$ cannot be dropped: for $M = 1$, a design with two distinct input values that are symmetric about the centre of the interval gives $Q_d = 0$.
\end{proof}

\begin{proof}[Proof of the decomposition \eqref{eq:deficit}]
Write $\vhat(g_d) = (N-1)^{-1}\bigl\{\sum_{i=1}^{N}[g_d]_i^2 - N\bar{g}_d^2\bigr\}$ with $\bar{g}_d = N^{-1}\one_N\T g_d$. By Lemma~\ref{lem:g}, $[g_d]_i$ and $\bar{g}_d$ have mean zero, with $\Ez\bigl[[g_d]_i^2\bigr] = [\Kt_d]_{ii}$ and $\Ez[\bar{g}_d^2] = \Var_z(\bar{g}_d) = N^{-2}\one_N\T\Kt_d\one_N$. Taking the expectation term by term gives \eqref{eq:deficit}, in agreement with \eqref{eq:Qsplit}. Splitting $\one_N\T\Kt_d\one_N$ into its diagonal and off-diagonal sums gives the equivalent form
\[
  Q_d(\rho_d)
  = \frac{1}{N}\sum_{i=1}^{N}[\Kt_d]_{ii}
    - \frac{1}{N(N-1)}\sum_{i \neq j}[\Kt_d]_{ij},
\]
the average marginal variance minus the average covariance between two distinct observed inputs. When the diagonal of $\Kt_d$ is one, $Q_d < 1$ if and only if this average covariance is positive. Without a condition on the sign of the covariances, positive semi-definiteness of $\Kt_d$ gives only $Q_d \le \tr(\Kt_d)/(N-1)$, which is $N/(N-1)$ for a unit diagonal.
\end{proof}

\begin{proof}[Proof of the component calibration \eqref{eq:calib-comp}]
Conditionally on $X$, $\rho_d$ and $\alpha_d$, the factor $\alpha_d/\sqrt{Q_d}$ is a deterministic scalar, and $\vhat$ is homogeneous of degree two, since
$\vhat(cv) = (N-1)^{-1}(cv)\T P_N(cv) = c^2\vhat(v)$. Hence
\[
  \Ez\bigl[\vhat(f_d)\bigr]
  = \frac{\alpha_d^2}{Q_d}\,\Ez\bigl[\vhat(g_d)\bigr]
  = \frac{\alpha_d^2}{Q_d}\cdot Q_d
  = \alpha_d^2
\]
by the definition \eqref{eq:Qdef}. Strict positivity of $Q_d$ holds under the conditions of Lemma~\ref{thm:Q}.
\end{proof}

\begin{proof}[Proof of the additive calibration \eqref{eq:calib-total}]
Let $\widehat{\Cov}_N(u,v) = (N-1)^{-1}u\T P_N v$, which is bilinear and symmetric with $\vhat(u) = \widehat{\Cov}_N(u,u)$. Bilinearity gives the exact decomposition
\[
  \vhat\Bigl(\sum_d f_d\Bigr)
  = \sum_{d=1}^{D}\vhat(f_d)
    + 2\sum_{1\le d<e\le D}\widehat{\Cov}_N(f_d,f_e).
\]
For $d \ne e$, cyclicity of the trace as in \eqref{eq:quadexp} gives
\[
  \Ez\bigl[\widehat{\Cov}_N(f_d,f_e)\bigr]
  = \frac{1}{N-1}\tr\bigl(P_N\,\Ez[f_ef_d\T]\bigr) = 0,
\]
since given $X$, $\rho$ and $\alpha$ the vectors $f_d$ and $f_e$ are functions of the independent $z_d$ and $z_e$ alone and each has mean zero, so $\Ez[f_ef_d\T] = \Ez[f_e]\Ez[f_d]\T = 0$. Applying \eqref{eq:calib-comp} to each diagonal term gives $\sum_d\alpha_d^2$, which equals $\alpha^2\sum_d\psi_d = \alpha^2$ because $\psi$ lies in the simplex.
\end{proof}

\clearpage

\section{Additional Simulation Results}
\label{app:simulation}

This appendix gives the full results of the simulation study of Section~\ref{sec:simulation}. The settings are those of Section~\ref{sec:sim-setup}. In the figures and tables, $\rho$ in a panel title or in the first column of a table is the generating length scale $\rho_0$, and RHS without a qualifier and RHS ($p_0{=}5$) are the calibration with $p_0 = \min(5, D-1)$.

\paragraph{Tables.}
For each DGP and kernel, three tables give the coverage of the central $90\%$ posterior intervals, the RMSE of the posterior mean relative to R2-GP, and the sampling diagnostics. All entries are means over $50$ replications, with one column for each $D$. For $\alpha$ and $\rho$, coverage and RMSE are taken across the components. A relative RMSE above one indicates worse recovery than under R2-GP. For each fit, $\widehat{R}$ is the largest value and ESS is the smallest bulk effective sample size over all $\alpha_d$, $\rho_d$, $\sigma$ and the log density, and divergences are counted among the $4000$ post-warm-up draws. Bold marks the best value in each column for each $\rho_0$.

\subsection{Dense DGP}
\label{app:sim-dense}

Figures~\ref{fig:dense-elpd-se} to \ref{fig:dense-elpd-m52} show the mean $\Delta\mathrm{elpd}$ and Figures~\ref{fig:dense-r2-se} to \ref{fig:dense-r2-m52} the posterior of $\Remp$, for the three kernels and the three length scales. Figure~\ref{fig:sim-r2} is the condition with the SE kernel and $\rho_0 = 0.3$. Tables~\ref{tab:coverage-dense-se} to \ref{tab:sampling-dense-m52} give the parameter recovery and the sampling diagnostics. Averaged over replications, the posterior mean of $\Remp$ minus $\Rsq_{\mathrm{true}}$ lies in $[-0.05, 0.00]$ under R2-GP in all conditions. At $D = 100$, it lies in $[0.62, 0.64]$ under the normal prior, in $[0.23, 0.62]$ under NG and in $[-0.27, -0.21]$ under RHS ($p_0{=}5$). The recovery of $\rho$ is governed by its informative prior: its coverage is $100\%$ after rounding and its RMSE is below $0.03$ under all priors.

\begin{figure}[!htbp]
  \centering
  \resizebox{0.85\textwidth}{!}{\input{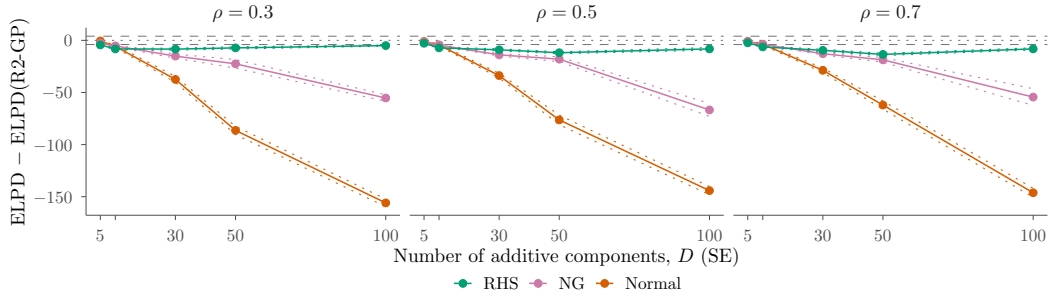}}
  \caption{Predictive Performance, Dense DGP, SE Kernel. Mean $\Delta\mathrm{elpd}$ over $50$ replications (solid) with $\pm 1$ Monte Carlo standard error (dotted), for each generating length scale $\rho_0$ (panels). Dashed lines mark $\pm 4$ and negative values favour R2-GP.}
  \label{fig:dense-elpd-se}
\end{figure}

\begin{figure}[!htbp]
  \centering
  \resizebox{0.85\textwidth}{!}{\input{figures/tikz/appendix/m32/elpd_difference_paper_dense_m32}}
  \caption{Predictive Performance, Dense DGP, Mat\'ern $3/2$ Kernel. As in Figure~\ref{fig:dense-elpd-se}.}
  \label{fig:dense-elpd-m32}
\end{figure}

\begin{figure}[!htbp]
  \centering
  \resizebox{0.85\textwidth}{!}{\input{figures/tikz/appendix/m52/elpd_difference_paper_dense_m52}}
  \caption{Predictive Performance, Dense DGP, Mat\'ern $5/2$ Kernel. As in Figure~\ref{fig:dense-elpd-se}.}
  \label{fig:dense-elpd-m52}
\end{figure}

\begin{figure}[!htbp]
  \centering
  \includegraphics[width=0.7\linewidth]{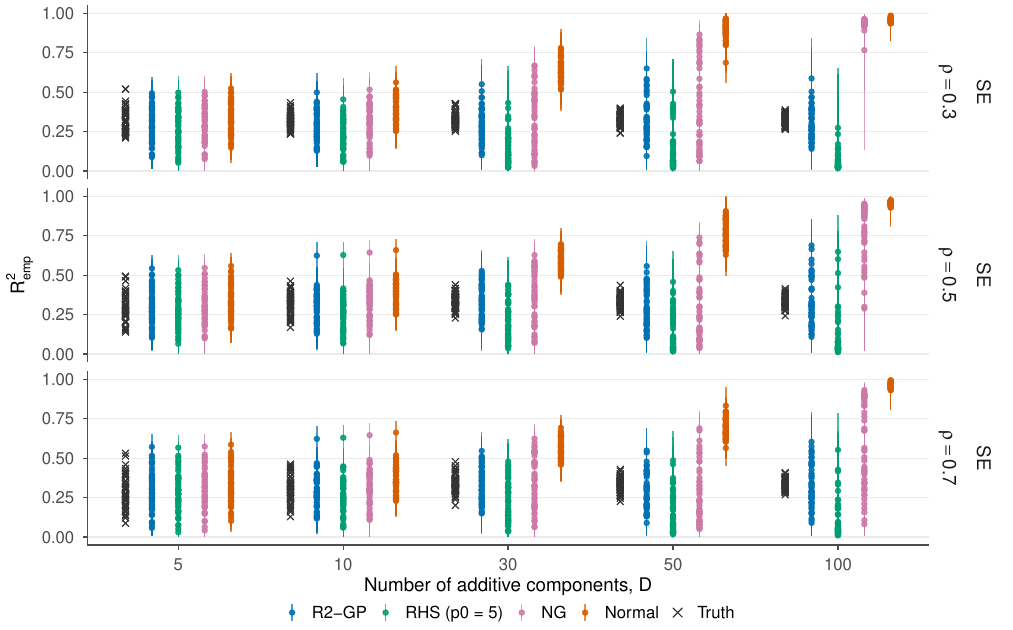}
  \caption{Explained Variance, Dense DGP, SE Kernel. Posterior means (points) and central $90\%$ intervals (bars) of $\Remp$ for each replication, for each generating length scale $\rho_0$ (rows). Crosses mark $\Rsq_{\mathrm{true}}$.}
  \label{fig:dense-r2-se}
\end{figure}

\begin{figure}[!htbp]
  \centering
  \includegraphics[width=0.7\linewidth]{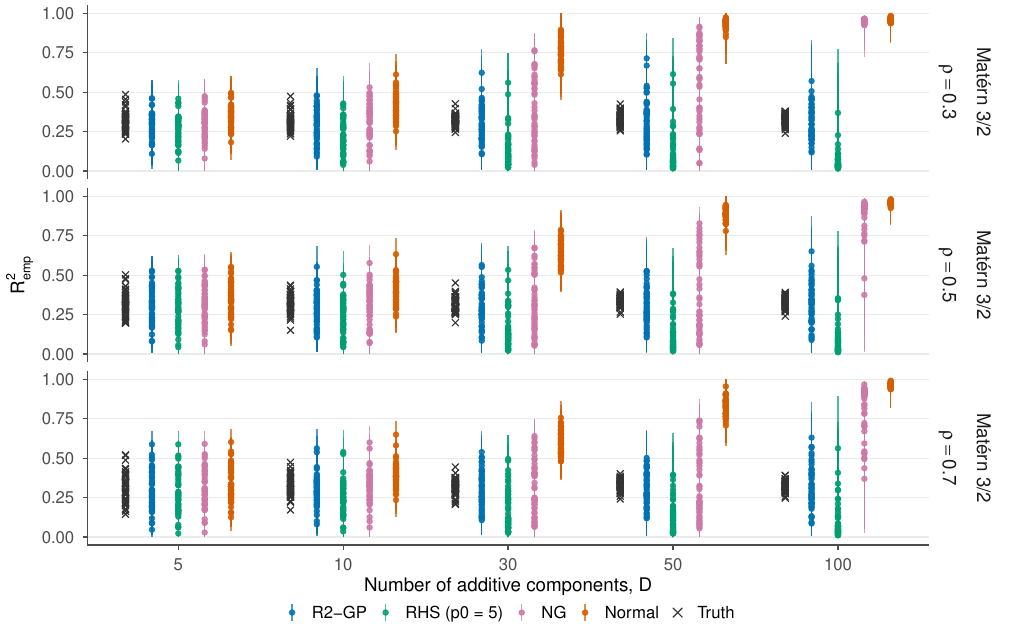}
  \caption{Explained Variance, Dense DGP, Mat\'ern $3/2$ Kernel. As in Figure~\ref{fig:dense-r2-se}.}
  \label{fig:dense-r2-m32}
\end{figure}

\begin{figure}[!htbp]
  \centering
  \includegraphics[width=0.7\linewidth]{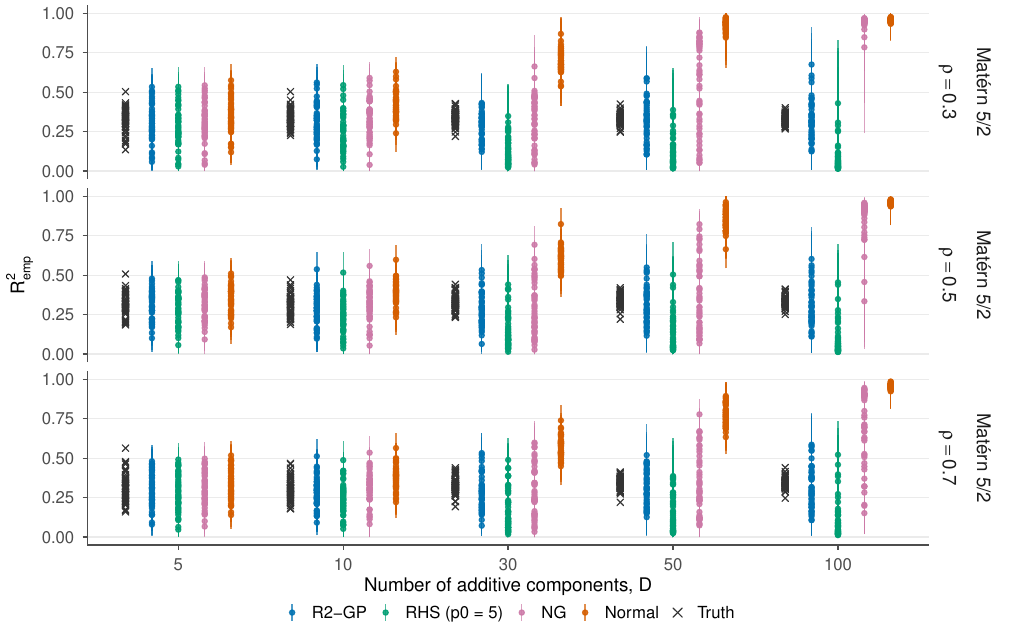}
  \caption{Explained Variance, Dense DGP, Mat\'ern $5/2$ Kernel. As in Figure~\ref{fig:dense-r2-se}.}
  \label{fig:dense-r2-m52}
\end{figure}

\clearpage

\begin{table}[p]
\apptable
\footnotesize
\setlength{\tabcolsep}{3pt}
\caption{Interval Coverage (\%), Dense DGP, SE Kernel. Bold marks the coverage closest to $90\%$.}\label{tab:coverage-dense-se}
\begin{tabular}{ll *{15}{r}}
\toprule
 &  & \multicolumn{5}{c}{$\alpha$} & \multicolumn{5}{c}{$\rho$} & \multicolumn{5}{c}{$\sigma$} \\
\cmidrule(lr){3-7} \cmidrule(lr){8-12} \cmidrule(lr){13-17}
$\rho$ & Prior & 5 & 10 & 30 & 50 & 100 & 5 & 10 & 30 & 50 & 100 & 5 & 10 & 30 & 50 & 100 \\
\midrule
0.3 & R2-GP & 97 & 100 & 100 & 99 & 100 & \textbf{100} & \textbf{100} & \textbf{100} & \textbf{100} & \textbf{100} & 92 & \textbf{82} & \textbf{92} & \textbf{94} & \textbf{98} \\
 & Normal & 95 & 96 & \textbf{95} & \textbf{96} & 100 & \textbf{100} & \textbf{100} & \textbf{100} & \textbf{100} & \textbf{100} & \textbf{90} & 80 & 38 & 2 & 0 \\
 & NG & \textbf{91} & 97 & 80 & 77 & \textbf{99} & \textbf{100} & \textbf{100} & \textbf{100} & \textbf{100} & \textbf{100} & 92 & 80 & 72 & 68 & 2 \\
 & RHS ($p_0{=}5$) & 84 & \textbf{85} & 63 & 47 & 18 & \textbf{100} & \textbf{100} & \textbf{100} & \textbf{100} & \textbf{100} & 92 & 74 & 64 & 54 & 42 \\
\addlinespace
0.5 & R2-GP & 99 & 99 & 100 & 100 & 100 & \textbf{100} & \textbf{100} & \textbf{100} & \textbf{100} & \textbf{100} & 94 & \textbf{86} & \textbf{88} & \textbf{90} & \textbf{92} \\
 & Normal & 95 & 95 & 94 & \textbf{94} & \textbf{99} & \textbf{100} & \textbf{100} & \textbf{100} & \textbf{100} & \textbf{100} & \textbf{92} & 84 & 60 & 8 & 0 \\
 & NG & 94 & 97 & \textbf{90} & 64 & \textbf{99} & \textbf{100} & \textbf{100} & \textbf{100} & \textbf{100} & \textbf{100} & 94 & 84 & 72 & 66 & 34 \\
 & RHS ($p_0{=}5$) & \textbf{88} & \textbf{88} & 82 & 52 & 33 & \textbf{100} & \textbf{100} & \textbf{100} & \textbf{100} & \textbf{100} & 94 & 84 & 70 & 58 & 48 \\
\addlinespace
0.7 & R2-GP & 98 & 99 & 100 & 100 & 99 & \textbf{100} & \textbf{100} & \textbf{100} & \textbf{100} & \textbf{100} & 92 & \textbf{90} & \textbf{88} & \textbf{88} & \textbf{96} \\
 & Normal & 95 & 95 & 94 & \textbf{95} & 96 & \textbf{100} & \textbf{100} & \textbf{100} & \textbf{100} & \textbf{100} & \textbf{90} & 84 & 72 & 24 & 0 \\
 & NG & 98 & 98 & \textbf{93} & 68 & \textbf{88} & \textbf{100} & \textbf{100} & \textbf{100} & \textbf{100} & \textbf{100} & 92 & 86 & 74 & 72 & 58 \\
 & RHS ($p_0{=}5$) & \textbf{87} & \textbf{89} & 85 & 57 & 40 & \textbf{100} & \textbf{100} & \textbf{100} & \textbf{100} & \textbf{100} & 92 & 86 & 74 & 68 & 54 \\
\bottomrule
\end{tabular}
\end{table}

\begin{table}[p]
\apptable
\footnotesize
\setlength{\tabcolsep}{3pt}
\caption{Relative RMSE, Dense DGP, SE Kernel. The R2-GP rows give its RMSE in italics.}\label{tab:rmse-dense-se}
\resizebox{\textwidth}{!}{%
\begin{tabular}{ll *{15}{r}}
\toprule
 &  & \multicolumn{5}{c}{$\alpha$} & \multicolumn{5}{c}{$\rho$} & \multicolumn{5}{c}{$\sigma$} \\
\cmidrule(lr){3-7} \cmidrule(lr){8-12} \cmidrule(lr){13-17}
$\rho$ & Prior & 5 & 10 & 30 & 50 & 100 & 5 & 10 & 30 & 50 & 100 & 5 & 10 & 30 & 50 & 100 \\
\midrule
0.3 & R2-GP & \textbf{\textit{0.097}} & \textbf{\textit{0.062}} & \textbf{\textit{0.033}} & \textbf{\textit{0.025}} & \textbf{\textit{0.017}} & \textit{0.022} & \textit{0.016} & \textit{0.008} & \textit{0.006} & \textit{0.003} & \textit{0.054} & \textit{0.084} & \textbf{\textit{0.079}} & \textbf{\textit{0.082}} & \textbf{\textit{0.078}} \\
 & Normal & 1.57 & 2.30 & 4.94 & 7.59 & 11.61 & 1.05 & 1.22 & 1.84 & 2.50 & 3.13 & \textbf{0.98} & \textbf{0.95} & 2.65 & 7.73 & 9.97 \\
 & NG & 1.52 & 1.88 & 2.74 & 2.92 & 3.13 & 0.94 & 0.98 & 1.12 & 1.30 & 3.13 & 1.01 & 1.02 & 1.66 & 2.30 & 9.26 \\
 & RHS ($p_0{=}5$) & 1.64 & 1.97 & 2.64 & 2.93 & 3.46 & \textbf{0.87} & \textbf{0.76} & \textbf{0.57} & \textbf{0.54} & \textbf{0.58} & 1.11 & 1.30 & 1.72 & 1.90 & 2.22 \\
\addlinespace
0.5 & R2-GP & \textbf{\textit{0.092}} & \textbf{\textit{0.061}} & \textbf{\textit{0.030}} & \textbf{\textit{0.029}} & \textbf{\textit{0.018}} & \textit{0.016} & \textit{0.012} & \textit{0.006} & \textit{0.004} & \textit{0.003} & \textit{0.053} & \textbf{\textit{0.065}} & \textbf{\textit{0.077}} & \textbf{\textit{0.078}} & \textbf{\textit{0.097}} \\
 & Normal & 1.96 & 2.85 & 6.65 & 8.05 & 13.26 & 1.01 & 1.07 & 1.60 & 2.50 & 3.65 & \textbf{0.98} & 1.02 & 1.83 & 4.79 & 7.92 \\
 & NG & 1.66 & 2.02 & 3.05 & 2.89 & 3.35 & 0.93 & 0.95 & 1.06 & 1.13 & 2.70 & 1.00 & 1.03 & 1.49 & 1.65 & 5.26 \\
 & RHS ($p_0{=}5$) & 1.75 & 2.02 & 2.61 & 2.51 & 3.08 & \textbf{0.88} & \textbf{0.81} & \textbf{0.68} & \textbf{0.66} & \textbf{0.81} & 1.07 & 1.21 & 1.48 & 1.71 & 1.88 \\
\addlinespace
0.7 & R2-GP & \textbf{\textit{0.095}} & \textbf{\textit{0.061}} & \textbf{\textit{0.031}} & \textbf{\textit{0.027}} & \textbf{\textit{0.018}} & \textit{0.011} & \textit{0.008} & \textit{0.004} & \textit{0.003} & \textit{0.002} & \textit{0.053} & \textbf{\textit{0.065}} & \textbf{\textit{0.074}} & \textbf{\textit{0.075}} & \textbf{\textit{0.082}} \\
 & Normal & 2.07 & 3.11 & 7.12 & 8.98 & 15.52 & 0.94 & 1.08 & 1.45 & 1.98 & 6.32 & \textbf{0.98} & \textbf{1.00} & 1.41 & 3.23 & 9.53 \\
 & NG & 1.71 & 2.05 & 3.02 & 3.04 & 3.50 & 0.92 & 0.96 & 1.01 & 1.27 & 1.82 & 1.00 & 1.03 & 1.41 & 1.52 & 3.34 \\
 & RHS ($p_0{=}5$) & 1.79 & 2.06 & 2.48 & 2.55 & 3.04 & \textbf{0.87} & \textbf{0.84} & \textbf{0.70} & \textbf{0.71} & \textbf{0.78} & 1.05 & 1.17 & 1.42 & 1.72 & 1.97 \\
\bottomrule
\end{tabular}
}
\end{table}

\begin{table}[p]
\apptable
\scriptsize
\setlength{\tabcolsep}{2.5pt}
\caption{Sampling Diagnostics, Dense DGP, SE Kernel. Time is the sampling time in seconds.}\label{tab:sampling-dense-se}
\begin{tabular}{ll *{20}{r}}
\toprule
 &  & \multicolumn{5}{c}{$\widehat{R}$} & \multicolumn{5}{c}{ESS} & \multicolumn{5}{c}{Divergences} & \multicolumn{5}{c}{Time (s)} \\
\cmidrule(lr){3-7} \cmidrule(lr){8-12} \cmidrule(lr){13-17} \cmidrule(lr){18-22}
$\rho$ & Prior & 5 & 10 & 30 & 50 & 100 & 5 & 10 & 30 & 50 & 100 & 5 & 10 & 30 & 50 & 100 & 5 & 10 & 30 & 50 & 100 \\
\midrule
0.3 & R2-GP & \textbf{1.00} & \textbf{1.00} & \textbf{1.01} & \textbf{1.01} & \textbf{1.01} & \textbf{1078} & \textbf{998} & \textbf{814} & \textbf{661} & 498 & \textbf{0.0} & \textbf{0.0} & \textbf{0.2} & \textbf{0.6} & \textbf{1.9} & \textbf{5.7} & \textbf{9.1} & \textbf{26} & \textbf{51} & \textbf{129} \\
 & Normal & \textbf{1.00} & 1.01 & \textbf{1.01} & 1.26 & 1.37 & 951 & 868 & 480 & 43 & 18 & 1.2 & 1.4 & 5.8 & 283 & 235 & 7.7 & 16 & 76 & 458 & 1634 \\
 & NG & 1.02 & 1.04 & 1.17 & 1.40 & 1.12 & 569 & 222 & 39 & 25 & 62 & 14 & 10 & 116 & 116 & 78 & 12 & 32 & 204 & 350 & 903 \\
 & RHS ($p_0{=}5$) & \textbf{1.00} & 1.01 & \textbf{1.01} & \textbf{1.01} & \textbf{1.01} & 896 & 770 & 551 & 578 & \textbf{708} & 2.3 & 3.6 & 7.2 & 15 & 12 & 8.4 & 14 & 36 & 63 & 134 \\
\addlinespace
0.5 & R2-GP & \textbf{1.00} & \textbf{1.00} & \textbf{1.01} & \textbf{1.01} & \textbf{1.01} & \textbf{1176} & \textbf{1148} & \textbf{1019} & \textbf{946} & \textbf{607} & \textbf{0.6} & \textbf{0.0} & \textbf{0.0} & \textbf{0.3} & \textbf{1.4} & \textbf{7.2} & \textbf{12} & \textbf{32} & \textbf{56} & 191 \\
 & Normal & \textbf{1.00} & \textbf{1.00} & \textbf{1.01} & 1.05 & 1.37 & 1043 & 973 & 756 & 217 & 19 & 2.0 & 1.6 & 1.1 & 80 & 233 & 9.3 & 20 & 97 & 349 & 1292 \\
 & NG & 1.01 & 1.03 & 1.15 & 1.25 & 1.43 & 605 & 328 & 51 & 24 & 32 & 8.9 & 18 & 69 & 199 & 49 & 13 & 31 & 175 & 395 & 961 \\
 & RHS ($p_0{=}5$) & \textbf{1.00} & 1.01 & \textbf{1.01} & \textbf{1.01} & \textbf{1.01} & 1035 & 911 & 578 & 624 & 568 & 4.1 & 4.6 & 17 & 16 & 32 & 11 & 16 & 45 & 69 & \textbf{171} \\
\addlinespace
0.7 & R2-GP & \textbf{1.00} & \textbf{1.00} & \textbf{1.00} & \textbf{1.01} & \textbf{1.01} & \textbf{1289} & \textbf{1266} & \textbf{1099} & \textbf{999} & \textbf{684} & \textbf{0.7} & \textbf{0.4} & \textbf{0.0} & \textbf{0.0} & \textbf{1.7} & \textbf{8.1} & \textbf{14} & \textbf{38} & \textbf{65} & \textbf{196} \\
 & Normal & \textbf{1.00} & \textbf{1.00} & 1.01 & \textbf{1.01} & 1.48 & 1111 & 1078 & 919 & 531 & 14 & 2.9 & 1.9 & 0.8 & 8.2 & 247 & 13 & 26 & 128 & 324 & 1457 \\
 & NG & 1.01 & 1.03 & 1.15 & 1.23 & 1.61 & 609 & 334 & 54 & 22 & 20 & 28 & 34 & 89 & 193 & 58 & 17 & 41 & 205 & 439 & 1125 \\
 & RHS ($p_0{=}5$) & \textbf{1.00} & 1.01 & 1.01 & \textbf{1.01} & \textbf{1.01} & 1093 & 930 & 632 & 636 & 587 & 6.6 & 10 & 12 & 18 & 21 & 13 & 21 & 57 & 82 & 199 \\
\bottomrule
\end{tabular}
\end{table}

\clearpage

\begin{table}[p]
\apptable
\footnotesize
\setlength{\tabcolsep}{3pt}
\caption{Interval Coverage (\%), Dense DGP, Mat\'ern $3/2$ Kernel. As in Table~\ref{tab:coverage-dense-se}.}\label{tab:coverage-dense-m32}
\begin{tabular}{ll *{15}{r}}
\toprule
 &  & \multicolumn{5}{c}{$\alpha$} & \multicolumn{5}{c}{$\rho$} & \multicolumn{5}{c}{$\sigma$} \\
\cmidrule(lr){3-7} \cmidrule(lr){8-12} \cmidrule(lr){13-17}
$\rho$ & Prior & 5 & 10 & 30 & 50 & 100 & 5 & 10 & 30 & 50 & 100 & 5 & 10 & 30 & 50 & 100 \\
\midrule
0.3 & R2-GP & 98 & 99 & 100 & 100 & \textbf{100} & \textbf{100} & \textbf{100} & \textbf{100} & \textbf{100} & \textbf{100} & \textbf{92} & \textbf{88} & \textbf{96} & \textbf{90} & \textbf{100} \\
 & Normal & 94 & 96 & 95 & 98 & \textbf{100} & \textbf{100} & \textbf{100} & \textbf{100} & \textbf{100} & \textbf{100} & 98 & 84 & 2 & 0 & 0 \\
 & NG & \textbf{92} & \textbf{95} & \textbf{87} & \textbf{91} & \textbf{100} & \textbf{100} & \textbf{100} & \textbf{100} & \textbf{100} & \textbf{100} & 96 & 82 & \textbf{84} & 52 & 0 \\
 & RHS ($p_0{=}5$) & 87 & 78 & 60 & 42 & 16 & \textbf{100} & \textbf{100} & \textbf{100} & \textbf{100} & \textbf{100} & 82 & 82 & 70 & 50 & 46 \\
\addlinespace
0.5 & R2-GP & 98 & 99 & 98 & 100 & \textbf{98} & \textbf{100} & \textbf{100} & \textbf{100} & \textbf{100} & \textbf{100} & \textbf{88} & 94 & \textbf{94} & \textbf{88} & \textbf{92} \\
 & Normal & 94 & \textbf{94} & \textbf{95} & \textbf{96} & 100 & \textbf{100} & \textbf{100} & \textbf{100} & \textbf{100} & \textbf{100} & 86 & \textbf{88} & 36 & 0 & 0 \\
 & NG & 93 & 95 & 82 & 77 & 99 & \textbf{100} & \textbf{100} & \textbf{100} & \textbf{100} & \textbf{100} & 86 & 86 & 80 & 66 & 14 \\
 & RHS ($p_0{=}5$) & \textbf{89} & 81 & 63 & 52 & 29 & \textbf{100} & \textbf{100} & \textbf{100} & \textbf{100} & \textbf{100} & 86 & 80 & 68 & 44 & 54 \\
\addlinespace
0.7 & R2-GP & 97 & 97 & 98 & 100 & \textbf{98} & \textbf{100} & \textbf{100} & \textbf{100} & \textbf{100} & \textbf{100} & 88 & \textbf{92} & \textbf{98} & \textbf{88} & \textbf{94} \\
 & Normal & \textbf{93} & 94 & \textbf{95} & \textbf{94} & 99 & \textbf{100} & \textbf{100} & \textbf{100} & \textbf{100} & \textbf{100} & \textbf{90} & 84 & 50 & 0 & 0 \\
 & NG & 95 & 96 & 84 & 75 & 99 & \textbf{100} & \textbf{100} & \textbf{100} & \textbf{100} & \textbf{100} & 88 & 86 & \textbf{82} & 70 & 30 \\
 & RHS ($p_0{=}5$) & \textbf{87} & \textbf{87} & 67 & 61 & 35 & \textbf{100} & \textbf{100} & \textbf{100} & \textbf{100} & \textbf{100} & 86 & 86 & 74 & 60 & 58 \\
\bottomrule
\end{tabular}
\end{table}

\begin{table}[p]
\apptable
\footnotesize
\setlength{\tabcolsep}{3pt}
\caption{Relative RMSE, Dense DGP, Mat\'ern $3/2$ Kernel. As in Table~\ref{tab:rmse-dense-se}.}\label{tab:rmse-dense-m32}
\resizebox{\textwidth}{!}{%
\begin{tabular}{ll *{15}{r}}
\toprule
 &  & \multicolumn{5}{c}{$\alpha$} & \multicolumn{5}{c}{$\rho$} & \multicolumn{5}{c}{$\sigma$} \\
\cmidrule(lr){3-7} \cmidrule(lr){8-12} \cmidrule(lr){13-17}
$\rho$ & Prior & 5 & 10 & 30 & 50 & 100 & 5 & 10 & 30 & 50 & 100 & 5 & 10 & 30 & 50 & 100 \\
\midrule
0.3 & R2-GP & \textbf{\textit{0.087}} & \textbf{\textit{0.065}} & \textbf{\textit{0.034}} & \textbf{\textit{0.025}} & \textbf{\textit{0.016}} & \textit{0.019} & \textit{0.012} & \textit{0.006} & \textit{0.005} & \textit{0.003} & \textit{0.057} & \textbf{\textit{0.068}} & \textbf{\textit{0.067}} & \textbf{\textit{0.094}} & \textbf{\textit{0.087}} \\
 & Normal & 1.60 & 1.90 & 4.92 & 6.26 & 11.18 & 1.07 & 1.17 & 2.09 & 2.18 & 3.42 & \textbf{0.93} & \textbf{1.00} & 5.67 & 7.54 & 8.93 \\
 & NG & 1.58 & 1.75 & 2.73 & 2.82 & 2.99 & 0.96 & 0.93 & 1.22 & 1.70 & 2.44 & 0.99 & 1.05 & 1.93 & 3.13 & 8.62 \\
 & RHS ($p_0{=}5$) & 1.70 & 1.95 & 2.72 & 3.01 & 3.65 & \textbf{0.87} & \textbf{0.71} & \textbf{0.55} & \textbf{0.57} & \textbf{0.74} & 1.16 & 1.39 & 2.03 & 1.78 & 2.13 \\
\addlinespace
0.5 & R2-GP & \textbf{\textit{0.096}} & \textbf{\textit{0.066}} & \textbf{\textit{0.035}} & \textbf{\textit{0.024}} & \textbf{\textit{0.016}} & \textit{0.012} & \textit{0.008} & \textit{0.004} & \textit{0.003} & \textit{0.002} & \textbf{\textit{0.065}} & \textit{0.068} & \textbf{\textit{0.064}} & \textbf{\textit{0.089}} & \textbf{\textit{0.098}} \\
 & Normal & 1.77 & 2.19 & 4.85 & 8.62 & 12.48 & 0.97 & 1.03 & 1.50 & 2.60 & 4.45 & 1.02 & \textbf{0.94} & 3.40 & 6.66 & 7.85 \\
 & NG & 1.60 & 1.84 & 2.61 & 3.35 & 3.38 & 0.93 & 0.91 & 1.05 & 1.39 & 2.46 & 1.02 & 1.01 & 1.76 & 2.09 & 6.61 \\
 & RHS ($p_0{=}5$) & 1.67 & 1.93 & 2.52 & 3.07 & 3.47 & \textbf{0.88} & \textbf{0.77} & \textbf{0.63} & \textbf{0.67} & \textbf{0.96} & 1.09 & 1.31 & 1.91 & 1.98 & 1.74 \\
\addlinespace
0.7 & R2-GP & \textbf{\textit{0.097}} & \textbf{\textit{0.062}} & \textbf{\textit{0.036}} & \textbf{\textit{0.024}} & \textbf{\textit{0.017}} & \textit{0.007} & \textit{0.005} & \textit{0.003} & \textit{0.002} & \textbf{\textit{0.002}} & \textit{0.064} & \textit{0.066} & \textbf{\textit{0.059}} & \textbf{\textit{0.079}} & \textbf{\textit{0.093}} \\
 & Normal & 1.81 & 2.62 & 5.09 & 9.39 & 13.07 & 0.92 & 1.00 & 1.32 & 2.08 & 6.92 & \textbf{0.96} & \textbf{0.97} & 2.75 & 5.92 & 8.48 \\
 & NG & 1.58 & 1.92 & 2.57 & 3.31 & 3.31 & 0.89 & 0.91 & 1.15 & 1.27 & 2.11 & 0.99 & 0.99 & 1.60 & 1.89 & 5.85 \\
 & RHS ($p_0{=}5$) & 1.66 & 1.99 & 2.40 & 2.96 & 3.26 & \textbf{0.84} & \textbf{0.77} & \textbf{0.75} & \textbf{0.91} & 1.11 & 1.08 & 1.24 & 1.86 & 2.06 & 1.75 \\
\bottomrule
\end{tabular}
}
\end{table}

\begin{table}[p]
\apptable
\scriptsize
\setlength{\tabcolsep}{2.5pt}
\caption{Sampling Diagnostics, Dense DGP, Mat\'ern $3/2$ Kernel. As in Table~\ref{tab:sampling-dense-se}.}\label{tab:sampling-dense-m32}
\begin{tabular}{ll *{20}{r}}
\toprule
 &  & \multicolumn{5}{c}{$\widehat{R}$} & \multicolumn{5}{c}{ESS} & \multicolumn{5}{c}{Divergences} & \multicolumn{5}{c}{Time (s)} \\
\cmidrule(lr){3-7} \cmidrule(lr){8-12} \cmidrule(lr){13-17} \cmidrule(lr){18-22}
$\rho$ & Prior & 5 & 10 & 30 & 50 & 100 & 5 & 10 & 30 & 50 & 100 & 5 & 10 & 30 & 50 & 100 & 5 & 10 & 30 & 50 & 100 \\
\midrule
0.3 & R2-GP & \textbf{1.00} & \textbf{1.00} & \textbf{1.01} & \textbf{1.01} & \textbf{1.01} & \textbf{1053} & \textbf{990} & \textbf{716} & \textbf{581} & 475 & \textbf{0.0} & \textbf{0.0} & \textbf{0.1} & \textbf{2.4} & \textbf{6.1} & \textbf{7.7} & \textbf{13} & \textbf{38} & \textbf{72} & \textbf{215} \\
 & Normal & \textbf{1.00} & 1.01 & 1.07 & 1.25 & 1.37 & 925 & 807 & 171 & 23 & 16 & 0.6 & 0.4 & 114 & 243 & 247 & 9.1 & 20 & 177 & 738 & 2401 \\
 & NG & 1.02 & 1.05 & 1.19 & 1.39 & 1.08 & 512 & 197 & 37 & 28 & 58 & 2.3 & 12 & 69 & 66 & 53 & 14 & 49 & 272 & 450 & 1678 \\
 & RHS ($p_0{=}5$) & 1.01 & 1.01 & \textbf{1.01} & \textbf{1.01} & 1.02 & 861 & 748 & 541 & 579 & \textbf{678} & 1.2 & 2.3 & 6.9 & 18 & 28 & 11 & 19 & 54 & 99 & 273 \\
\addlinespace
0.5 & R2-GP & \textbf{1.00} & \textbf{1.00} & \textbf{1.01} & \textbf{1.01} & \textbf{1.01} & \textbf{1179} & \textbf{1120} & \textbf{966} & \textbf{775} & 540 & \textbf{0.0} & \textbf{0.1} & \textbf{0.1} & \textbf{0.4} & \textbf{7.2} & \textbf{8.7} & \textbf{15} & \textbf{42} & \textbf{73} & \textbf{265} \\
 & Normal & \textbf{1.00} & 1.01 & \textbf{1.01} & 1.20 & 1.39 & 1016 & 988 & 545 & 42 & 17 & 1.7 & 0.8 & 6.5 & 239 & 273 & 12 & 25 & 138 & 690 & 2543 \\
 & NG & 1.03 & 1.04 & 1.18 & 1.33 & 1.21 & 650 & 276 & 38 & 23 & 50 & 6.1 & 28 & 110 & 118 & 73 & 16 & 47 & 294 & 483 & 1760 \\
 & RHS ($p_0{=}5$) & \textbf{1.00} & 1.01 & \textbf{1.01} & \textbf{1.01} & 1.02 & 940 & 833 & 614 & 559 & \textbf{634} & 2.2 & 8.4 & 8.1 & 17 & 40 & 14 & 22 & 59 & 112 & 317 \\
\addlinespace
0.7 & R2-GP & \textbf{1.00} & \textbf{1.00} & \textbf{1.01} & \textbf{1.01} & \textbf{1.01} & \textbf{1228} & \textbf{1179} & \textbf{1099} & \textbf{856} & \textbf{610} & \textbf{1.0} & \textbf{0.0} & \textbf{0.1} & \textbf{0.4} & \textbf{2.0} & \textbf{11} & \textbf{17} & \textbf{49} & \textbf{90} & \textbf{282} \\
 & Normal & \textbf{1.00} & \textbf{1.00} & \textbf{1.01} & 1.09 & 1.45 & 1057 & 1026 & 754 & 112 & 16 & 2.1 & 1.0 & 2.3 & 130 & 238 & 14 & 30 & 148 & 574 & 2364 \\
 & NG & 1.01 & 1.03 & 1.15 & 1.28 & 1.37 & 679 & 317 & 41 & 21 & 31 & 18 & 24 & 122 & 118 & 59 & 20 & 49 & 294 & 568 & 1924 \\
 & RHS ($p_0{=}5$) & \textbf{1.00} & 1.01 & \textbf{1.01} & 1.02 & 1.02 & 1018 & 921 & 643 & 530 & 561 & 4.2 & 5.5 & 14 & 33 & 37 & 16 & 27 & 72 & 127 & 349 \\
\bottomrule
\end{tabular}
\end{table}

\clearpage

\begin{table}[p]
\apptable
\footnotesize
\setlength{\tabcolsep}{3pt}
\caption{Interval Coverage (\%), Dense DGP, Mat\'ern $5/2$ Kernel. As in Table~\ref{tab:coverage-dense-se}.}\label{tab:coverage-dense-m52}
\begin{tabular}{ll *{15}{r}}
\toprule
 &  & \multicolumn{5}{c}{$\alpha$} & \multicolumn{5}{c}{$\rho$} & \multicolumn{5}{c}{$\sigma$} \\
\cmidrule(lr){3-7} \cmidrule(lr){8-12} \cmidrule(lr){13-17}
$\rho$ & Prior & 5 & 10 & 30 & 50 & 100 & 5 & 10 & 30 & 50 & 100 & 5 & 10 & 30 & 50 & 100 \\
\midrule
0.3 & R2-GP & 92 & 97 & 100 & 100 & \textbf{100} & \textbf{100} & \textbf{100} & \textbf{100} & \textbf{100} & \textbf{100} & 82 & \textbf{88} & 96 & \textbf{92} & \textbf{98} \\
 & Normal & 93 & 97 & \textbf{95} & 97 & \textbf{100} & \textbf{100} & \textbf{100} & \textbf{100} & \textbf{100} & \textbf{100} & \textbf{88} & 86 & 20 & 0 & 0 \\
 & NG & \textbf{89} & \textbf{95} & \textbf{85} & \textbf{84} & \textbf{100} & \textbf{100} & \textbf{100} & \textbf{100} & \textbf{100} & \textbf{100} & 84 & 86 & \textbf{86} & 48 & 2 \\
 & RHS ($p_0{=}5$) & 81 & 80 & 59 & 49 & 22 & \textbf{100} & \textbf{100} & \textbf{100} & \textbf{100} & \textbf{100} & 82 & 86 & 68 & 64 & 58 \\
\addlinespace
0.5 & R2-GP & 98 & 98 & 98 & 99 & 100 & \textbf{100} & \textbf{100} & \textbf{100} & \textbf{100} & \textbf{100} & \textbf{90} & 92 & \textbf{96} & \textbf{92} & \textbf{94} \\
 & Normal & 94 & \textbf{95} & \textbf{95} & \textbf{94} & 100 & \textbf{100} & \textbf{100} & \textbf{100} & \textbf{100} & \textbf{100} & \textbf{90} & 88 & 48 & 2 & 0 \\
 & NG & 94 & 97 & 82 & 76 & \textbf{99} & \textbf{100} & \textbf{100} & \textbf{100} & \textbf{100} & \textbf{100} & 88 & \textbf{90} & 80 & 70 & 14 \\
 & RHS ($p_0{=}5$) & \textbf{88} & 80 & 63 & 56 & 27 & \textbf{100} & \textbf{100} & \textbf{100} & \textbf{100} & \textbf{100} & 86 & 88 & 62 & 60 & 50 \\
\addlinespace
0.7 & R2-GP & 98 & 98 & 98 & 100 & 100 & \textbf{100} & \textbf{100} & \textbf{100} & \textbf{100} & \textbf{100} & \textbf{88} & 94 & \textbf{90} & \textbf{94} & \textbf{94} \\
 & Normal & 93 & \textbf{95} & \textbf{95} & \textbf{94} & 99 & \textbf{100} & \textbf{100} & \textbf{100} & \textbf{100} & \textbf{100} & \textbf{88} & \textbf{90} & 60 & 8 & 0 \\
 & NG & 95 & 97 & 84 & 77 & \textbf{95} & \textbf{100} & \textbf{100} & \textbf{100} & \textbf{100} & \textbf{100} & \textbf{88} & 92 & 78 & 76 & 44 \\
 & RHS ($p_0{=}5$) & \textbf{90} & \textbf{85} & 66 & 61 & 38 & \textbf{100} & \textbf{100} & \textbf{100} & \textbf{100} & \textbf{100} & 86 & 88 & 66 & 66 & 54 \\
\bottomrule
\end{tabular}
\end{table}

\begin{table}[p]
\apptable
\footnotesize
\setlength{\tabcolsep}{3pt}
\caption{Relative RMSE, Dense DGP, Mat\'ern $5/2$ Kernel. As in Table~\ref{tab:rmse-dense-se}.}\label{tab:rmse-dense-m52}
\resizebox{\textwidth}{!}{%
\begin{tabular}{ll *{15}{r}}
\toprule
 &  & \multicolumn{5}{c}{$\alpha$} & \multicolumn{5}{c}{$\rho$} & \multicolumn{5}{c}{$\sigma$} \\
\cmidrule(lr){3-7} \cmidrule(lr){8-12} \cmidrule(lr){13-17}
$\rho$ & Prior & 5 & 10 & 30 & 50 & 100 & 5 & 10 & 30 & 50 & 100 & 5 & 10 & 30 & 50 & 100 \\
\midrule
0.3 & R2-GP & \textbf{\textit{0.100}} & \textbf{\textit{0.066}} & \textbf{\textit{0.031}} & \textbf{\textit{0.025}} & \textbf{\textit{0.017}} & \textit{0.020} & \textit{0.016} & \textit{0.007} & \textit{0.005} & \textit{0.003} & \textbf{\textit{0.066}} & \textbf{\textit{0.071}} & \textbf{\textit{0.056}} & \textbf{\textit{0.082}} & \textbf{\textit{0.078}} \\
 & Normal & 1.58 & 2.03 & 5.16 & 6.79 & 10.64 & 1.03 & 1.14 & 1.94 & 2.37 & 2.89 & \textbf{1.00} & 1.11 & 5.11 & 8.49 & 9.87 \\
 & NG & 1.53 & 1.76 & 2.81 & 3.02 & 2.78 & 0.93 & 0.99 & 1.04 & 1.78 & 2.63 & 1.02 & 1.06 & 1.66 & 3.29 & 9.56 \\
 & RHS ($p_0{=}5$) & 1.62 & 1.89 & 2.91 & 2.96 & 3.36 & \textbf{0.87} & \textbf{0.81} & \textbf{0.54} & \textbf{0.57} & \textbf{0.80} & 1.09 & 1.15 & 2.55 & 1.88 & 1.86 \\
\addlinespace
0.5 & R2-GP & \textbf{\textit{0.099}} & \textbf{\textit{0.067}} & \textbf{\textit{0.034}} & \textbf{\textit{0.024}} & \textbf{\textit{0.018}} & \textit{0.013} & \textit{0.009} & \textit{0.004} & \textit{0.003} & \textit{0.002} & \textbf{\textit{0.054}} & \textbf{\textit{0.063}} & \textbf{\textit{0.070}} & \textbf{\textit{0.086}} & \textbf{\textit{0.087}} \\
 & Normal & 1.81 & 2.20 & 5.18 & 8.96 & 11.94 & 0.95 & 1.09 & 1.58 & 2.44 & 4.71 & 1.02 & 1.03 & 2.31 & 6.07 & 8.97 \\
 & NG & 1.59 & 1.80 & 2.67 & 3.37 & 3.25 & 0.92 & 0.93 & 0.99 & 1.33 & 3.00 & 1.02 & 1.02 & 1.58 & 1.90 & 7.14 \\
 & RHS ($p_0{=}5$) & 1.67 & 1.89 & 2.55 & 3.01 & 3.19 & \textbf{0.85} & \textbf{0.78} & \textbf{0.64} & \textbf{0.67} & \textbf{0.85} & 1.04 & 1.13 & 1.82 & 1.73 & 1.75 \\
\addlinespace
0.7 & R2-GP & \textbf{\textit{0.101}} & \textbf{\textit{0.064}} & \textbf{\textit{0.036}} & \textbf{\textit{0.024}} & \textbf{\textit{0.017}} & \textit{0.009} & \textit{0.006} & \textit{0.003} & \textit{0.002} & \textit{0.002} & \textbf{\textit{0.053}} & \textbf{\textit{0.063}} & \textbf{\textit{0.072}} & \textbf{\textit{0.079}} & \textbf{\textit{0.087}} \\
 & Normal & 1.90 & 2.52 & 5.49 & 9.80 & 13.47 & 0.91 & 0.99 & 1.45 & 1.87 & 5.34 & 1.02 & \textbf{1.00} & 1.67 & 4.58 & 8.83 \\
 & NG & 1.62 & 1.89 & 2.62 & 3.42 & 3.38 & 0.89 & 0.90 & 1.07 & 1.36 & 2.01 & 1.01 & 1.03 & 1.40 & 1.75 & 4.77 \\
 & RHS ($p_0{=}5$) & 1.71 & 1.92 & 2.42 & 2.91 & 3.15 & \textbf{0.86} & \textbf{0.79} & \textbf{0.71} & \textbf{0.73} & \textbf{0.95} & 1.02 & 1.10 & 1.63 & 1.79 & 1.81 \\
\bottomrule
\end{tabular}
}
\end{table}

\begin{table}[p]
\apptable
\scriptsize
\setlength{\tabcolsep}{2.5pt}
\caption{Sampling Diagnostics, Dense DGP, Mat\'ern $5/2$ Kernel. As in Table~\ref{tab:sampling-dense-se}.}\label{tab:sampling-dense-m52}
\begin{tabular}{ll *{20}{r}}
\toprule
 &  & \multicolumn{5}{c}{$\widehat{R}$} & \multicolumn{5}{c}{ESS} & \multicolumn{5}{c}{Divergences} & \multicolumn{5}{c}{Time (s)} \\
\cmidrule(lr){3-7} \cmidrule(lr){8-12} \cmidrule(lr){13-17} \cmidrule(lr){18-22}
$\rho$ & Prior & 5 & 10 & 30 & 50 & 100 & 5 & 10 & 30 & 50 & 100 & 5 & 10 & 30 & 50 & 100 & 5 & 10 & 30 & 50 & 100 \\
\midrule
0.3 & R2-GP & \textbf{1.00} & \textbf{1.00} & \textbf{1.01} & \textbf{1.01} & \textbf{1.01} & \textbf{1053} & \textbf{984} & \textbf{765} & \textbf{635} & 448 & \textbf{0.0} & \textbf{0.0} & \textbf{0.2} & \textbf{0.4} & \textbf{4.6} & \textbf{7.2} & \textbf{12} & \textbf{33} & \textbf{59} & \textbf{178} \\
 & Normal & \textbf{1.00} & 1.01 & 1.02 & 1.28 & 1.34 & 932 & 839 & 320 & 30 & 17 & 0.8 & 0.6 & 38 & 233 & 263 & 9.2 & 18 & 135 & 662 & 2012 \\
 & NG & 1.01 & 1.03 & 1.21 & 1.32 & 1.10 & 547 & 265 & 27 & 29 & 58 & 15 & 15 & 104 & 102 & 65 & 16 & 40 & 275 & 397 & 1400 \\
 & RHS ($p_0{=}5$) & 1.01 & 1.01 & \textbf{1.01} & \textbf{1.01} & 1.03 & 860 & 736 & 584 & 558 & \textbf{575} & 2.1 & 2.7 & 7.3 & 12 & 47 & 11 & 18 & 49 & 84 & 242 \\
\addlinespace
0.5 & R2-GP & \textbf{1.00} & \textbf{1.00} & \textbf{1.01} & \textbf{1.01} & \textbf{1.01} & \textbf{1169} & \textbf{1158} & \textbf{996} & \textbf{809} & 618 & \textbf{0.0} & \textbf{0.0} & \textbf{0.1} & \textbf{0.7} & \textbf{1.6} & \textbf{9.0} & \textbf{14} & \textbf{39} & \textbf{69} & \textbf{207} \\
 & Normal & \textbf{1.00} & 1.01 & \textbf{1.01} & 1.14 & 1.38 & 990 & 975 & 677 & 84 & 17 & 1.8 & 2.0 & 1.9 & 136 & 212 & 12 & 23 & 128 & 529 & 2034 \\
 & NG & 1.01 & 1.03 & 1.17 & 1.28 & 1.22 & 644 & 266 & 38 & 23 & 48 & 7.2 & 28 & 106 & 104 & 84 & 16 & 44 & 283 & 525 & 1315 \\
 & RHS ($p_0{=}5$) & \textbf{1.00} & 1.01 & \textbf{1.01} & \textbf{1.01} & \textbf{1.01} & 971 & 896 & 654 & 583 & \textbf{657} & 2.3 & 4.0 & 15 & 22 & 29 & 13 & 21 & 57 & 99 & 238 \\
\addlinespace
0.7 & R2-GP & \textbf{1.00} & \textbf{1.00} & \textbf{1.01} & \textbf{1.01} & \textbf{1.01} & \textbf{1241} & \textbf{1231} & \textbf{1114} & \textbf{902} & \textbf{659} & \textbf{0.2} & \textbf{0.1} & \textbf{0.1} & \textbf{0.6} & \textbf{2.5} & \textbf{9.9} & \textbf{17} & \textbf{48} & \textbf{85} & \textbf{239} \\
 & Normal & \textbf{1.00} & \textbf{1.00} & \textbf{1.01} & 1.03 & 1.36 & 1102 & 1069 & 862 & 279 & 15 & 2.2 & 1.4 & 2.1 & 53 & 228 & 15 & 31 & 150 & 500 & 2106 \\
 & NG & 1.01 & 1.03 & 1.15 & 1.25 & 1.50 & 666 & 330 & 37 & 19 & 29 & 12 & 35 & 119 & 122 & 69 & 18 & 51 & 281 & 542 & 1488 \\
 & RHS ($p_0{=}5$) & 1.01 & 1.01 & \textbf{1.01} & \textbf{1.01} & \textbf{1.01} & 1010 & 964 & 717 & 565 & 593 & 10 & 3.9 & 14 & 16 & 28 & 15 & 24 & 68 & 128 & 310 \\
\bottomrule
\end{tabular}
\end{table}

\clearpage

\subsection{Sparse DGP}
\label{app:sim-sparse}

In the sparse DGP only the first five components are active, with $\alpha_d^2 = \sigma^2\Rsq_0/\{5(1-\Rsq_0)\}$ for $d \le 5$ and $\alpha_d = 0$ otherwise, so that the total signal variance equals that of the dense DGP \eqref{eq:sim-alpha}. All other settings are those of Section~\ref{sec:sim-setup}. The priors are R2-GP, normal, NG and RHS with $p_0 = \min(5, D-1)$, which equals the number of active components for $D \ge 10$. Figures~\ref{fig:sparse-elpd-se} to \ref{fig:sparse-r2-m52} and Tables~\ref{tab:coverage-sparse-se} to \ref{tab:sampling-sparse-m52} give the results. The tables report the RMSE of $\alpha$ separately for the active and the inactive components. For $D \ge 30$, the intervals of R2-GP contain the generating value of an active $\alpha_d$ in at most $27\%$ of cases, and those of RHS in at most $49\%$. This agrees with the uniform allocation for $a = 1$ that Section~\ref{sec:sim-results} discusses. The recovery of $\sigma$ and of $\Rsq_{\mathrm{true}}$ under R2-GP stays accurate.

\begin{figure}[!htbp]
  \centering
  \resizebox{0.85\textwidth}{!}{\input{figures/tikz/appendix/sparse/se/elpd_difference_paper_sparse_se}}
  \caption{Predictive Performance, Sparse DGP, SE Kernel. As in Figure~\ref{fig:dense-elpd-se}.}
  \label{fig:sparse-elpd-se}
\end{figure}

\begin{figure}[!htbp]
  \centering
  \resizebox{0.85\textwidth}{!}{\input{figures/tikz/appendix/sparse/m32/elpd_difference_paper_sparse_m32}}
  \caption{Predictive Performance, Sparse DGP, Mat\'ern $3/2$ Kernel. As in Figure~\ref{fig:dense-elpd-se}.}
  \label{fig:sparse-elpd-m32}
\end{figure}

\clearpage

\begin{figure}[!htbp]
  \centering
  \resizebox{0.85\textwidth}{!}{\input{figures/tikz/appendix/sparse/m52/elpd_difference_paper_sparse_m52}}
  \caption{Predictive Performance, Sparse DGP, Mat\'ern $5/2$ Kernel. As in Figure~\ref{fig:dense-elpd-se}.}
  \label{fig:sparse-elpd-m52}
\end{figure}

\begin{figure}[!htbp]
  \centering
  \includegraphics[width=0.7\linewidth]{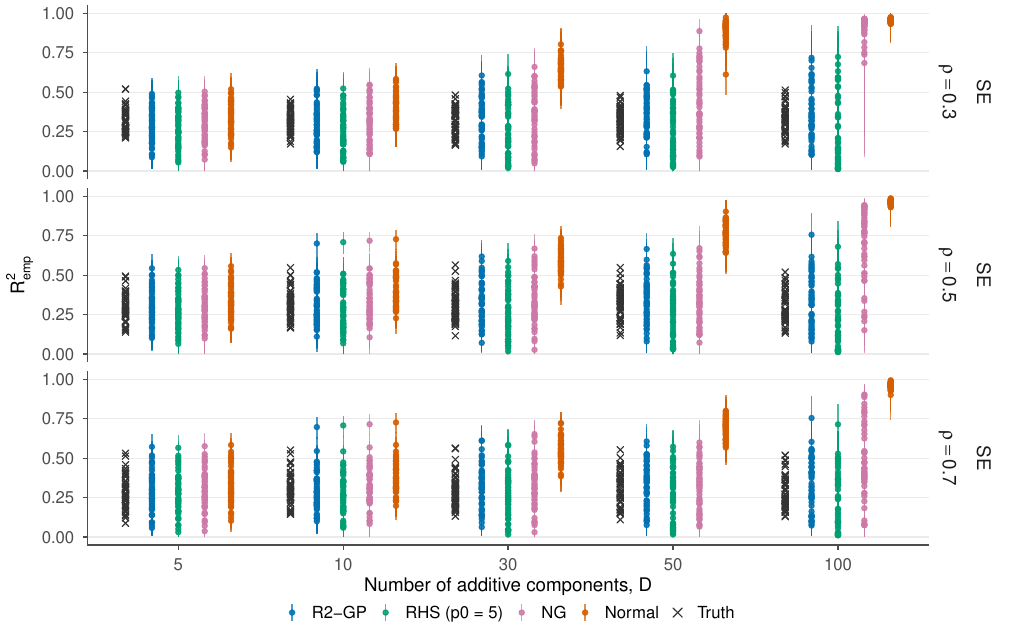}
  \caption{Explained Variance, Sparse DGP, SE Kernel. As in Figure~\ref{fig:dense-r2-se}.}
  \label{fig:sparse-r2-se}
\end{figure}

\begin{figure}[!htbp]
  \centering
  \includegraphics[width=0.7\linewidth]{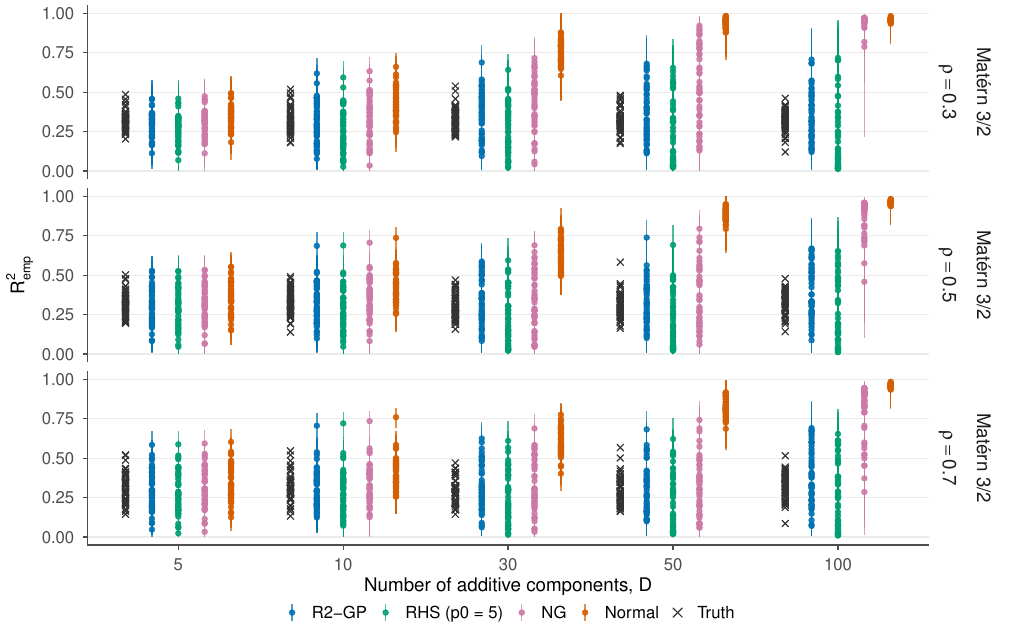}
  \caption{Explained Variance, Sparse DGP, Mat\'ern $3/2$ Kernel. As in Figure~\ref{fig:dense-r2-se}.}
  \label{fig:sparse-r2-m32}
\end{figure}

\begin{figure}[!htbp]
  \centering
  \includegraphics[width=0.7\linewidth]{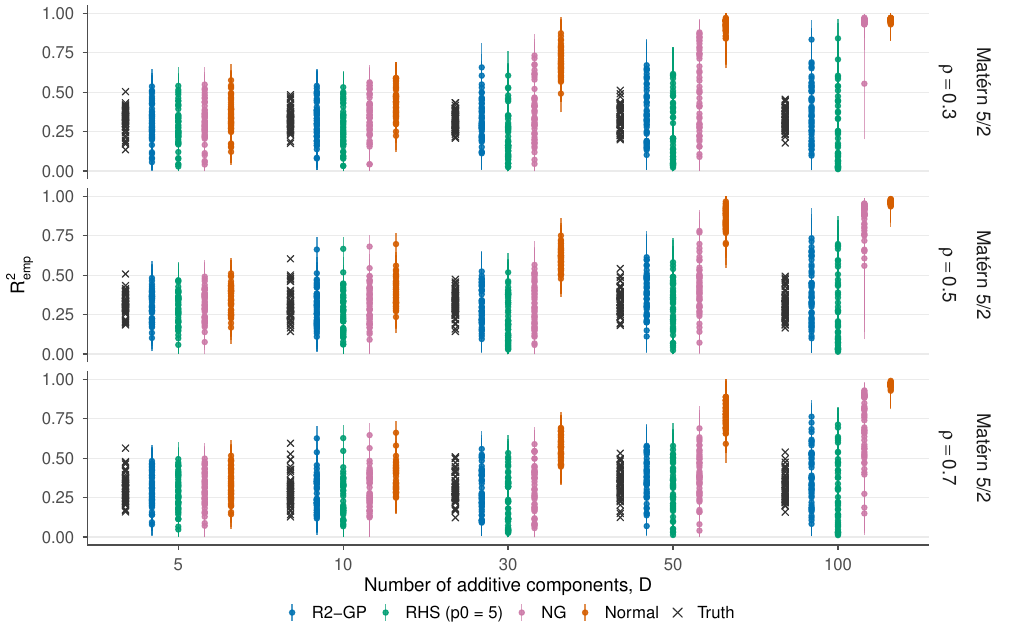}
  \caption{Explained Variance, Sparse DGP, Mat\'ern $5/2$ Kernel. As in Figure~\ref{fig:dense-r2-se}.}
  \label{fig:sparse-r2-m52}
\end{figure}

\clearpage

\newsavebox{\simtablebox}
\newcommand{\simtablesetup}[1]{%
  \setlength{\abovecaptionskip}{11pt}%
  \setlength{\belowcaptionskip}{11pt}%
  \let\simcaption\caption
  \renewcommand{\caption}[1]{\simcaption{#1}}%
  \renewenvironment{table*}[1][]{\begin{table}[p]}{\end{table}}}
\newcommand{\simtable}[2]{%
  \begingroup
  \simtablesetup{#1}%
  \input{#2}%
  \endgroup}
\newcommand{\simtablefitted}[2]{%
  \begingroup
  \simtablesetup{#1}%
  \let\simtabular\tabular
  \let\endsimtabular\endtabular
  \renewenvironment{tabular}%
    {\begin{lrbox}{\simtablebox}\simtabular}%
    {\endsimtabular\end{lrbox}\resizebox{\textwidth}{!}{\usebox{\simtablebox}}}%
  \input{#2}%
  \endgroup}

  \begingroup
  \simtablesetup{Interval Coverage (\%), Sparse DGP, SE Kernel. As in Table~\ref{tab:coverage-dense-se}, for the active components. The interval of a positive parameter cannot contain the zero of an inactive component.}%
  \begin{table*}[t]
\centering
\footnotesize
\setlength{\tabcolsep}{3pt}
\caption{Coverage (\%) of the nominal $90\%$ posterior interval (sparse DGP, squared exponential kernel), averaged over 50 replications. Columns give the number of additive components $D$. For $\alpha$ and $\rho$ it is the share of the active components whose interval contains the truth; the interval of a positive parameter cannot contain the exact zero of an inactive one. Bold marks the coverage closest to $90\%$ in each column within a length-scale. All models use the $N(\rho, 0.1)$ prior centred at the generating length-scale on the length-scales.}\label{tab:coverage-sparse-se}
\begin{tabular}{ll *{15}{r}}
\toprule
 &  & \multicolumn{5}{c}{$\alpha$} & \multicolumn{5}{c}{$\rho$} & \multicolumn{5}{c}{$\sigma$} \\
\cmidrule(lr){3-7} \cmidrule(lr){8-12} \cmidrule(lr){13-17}
$\rho$ & Prior & 5 & 10 & 30 & 50 & 100 & 5 & 10 & 30 & 50 & 100 & 5 & 10 & 30 & 50 & 100 \\
\midrule
0.3 & R2-GP & 97 & 77 & 23 & 11 & 1 & \textbf{100} & \textbf{100} & \textbf{100} & \textbf{100} & \textbf{100} & 92 & \textbf{82} & \textbf{84} & 94 & \textbf{94} \\
 & Normal & 95 & 96 & \textbf{95} & \textbf{95} & \textbf{100} & \textbf{100} & \textbf{100} & \textbf{100} & \textbf{100} & \textbf{100} & \textbf{90} & 78 & 30 & 2 & 0 \\
 & NG & \textbf{92} & \textbf{85} & 64 & 67 & 60 & \textbf{100} & \textbf{100} & \textbf{100} & \textbf{100} & \textbf{100} & 92 & 80 & 80 & 80 & 8 \\
 & RHS ($p_0{=}5$) & 85 & 66 & 44 & 41 & 23 & \textbf{100} & \textbf{100} & \textbf{100} & \textbf{100} & \textbf{100} & 92 & 78 & 80 & \textbf{90} & 84 \\
\addlinespace
0.5 & R2-GP & 99 & 84 & 21 & 9 & 0 & \textbf{100} & \textbf{100} & \textbf{100} & \textbf{100} & \textbf{100} & 94 & \textbf{84} & \textbf{92} & \textbf{90} & \textbf{90} \\
 & Normal & 95 & 95 & \textbf{97} & \textbf{96} & \textbf{98} & \textbf{100} & \textbf{100} & \textbf{100} & \textbf{100} & \textbf{100} & \textbf{92} & 78 & 52 & 8 & 0 \\
 & NG & 94 & \textbf{90} & 66 & 63 & 64 & \textbf{100} & \textbf{100} & \textbf{100} & \textbf{100} & \textbf{100} & 94 & 80 & \textbf{88} & 80 & 46 \\
 & RHS ($p_0{=}5$) & \textbf{88} & 75 & 47 & 40 & 25 & \textbf{100} & \textbf{100} & \textbf{100} & \textbf{100} & \textbf{100} & 94 & \textbf{84} & 94 & 82 & 94 \\
\addlinespace
0.7 & R2-GP & 98 & 82 & 18 & 6 & 1 & \textbf{100} & \textbf{100} & \textbf{100} & \textbf{100} & \textbf{100} & \textbf{92} & \textbf{88} & 82 & \textbf{90} & \textbf{90} \\
 & Normal & 94 & 95 & \textbf{97} & \textbf{94} & \textbf{99} & \textbf{100} & \textbf{100} & \textbf{100} & \textbf{100} & \textbf{100} & \textbf{92} & 82 & 68 & 18 & 0 \\
 & NG & 98 & \textbf{91} & 66 & 56 & 55 & \textbf{100} & \textbf{100} & \textbf{100} & \textbf{100} & \textbf{100} & \textbf{92} & 82 & 86 & \textbf{90} & 76 \\
 & RHS ($p_0{=}5$) & \textbf{88} & 76 & 45 & 38 & 26 & \textbf{100} & \textbf{100} & \textbf{100} & \textbf{100} & \textbf{100} & \textbf{92} & 84 & \textbf{88} & 86 & \textbf{90} \\
\bottomrule
\end{tabular}
\end{table*}
  \endgroup
\simtablefitted{Relative RMSE, Sparse DGP, SE Kernel. As in Table~\ref{tab:rmse-dense-se}, with the RMSE of $\alpha$ taken separately across the active and the inactive components, and the RMSE of $\rho$ across the active components. At $D = 5$ all components are active.}{tables/sparse/se/parameter_recovery}
  \begingroup
  \simtablesetup{Sampling Diagnostics, Sparse DGP, SE Kernel. As in Table~\ref{tab:sampling-dense-se}.}%
  \begin{table*}[t]
\centering
\scriptsize
\setlength{\tabcolsep}{2.5pt}
\caption{Sampling behaviour (sparse DGP, squared exponential kernel), each statistic averaged over 50 replications. Columns give the number of additive components $D$. $\widehat{R}$ is the largest and ESS the smallest bulk effective sample size over all $\alpha_d$, $\rho_d$, $\sigma$ and $\mathrm{lp}$ of a fit; divergences are out of 4000 post-warmup draws; time is the sampling time in seconds. Bold marks the best value: the $\widehat{R}$ closest to one, the largest ESS, the fewest divergences and the shortest time, in each column within a length-scale. All models use the $N(\rho, 0.1)$ prior centred at the generating length-scale on the length-scales.}\label{tab:sampling-sparse-se}
\begin{tabular}{ll *{20}{r}}
\toprule
 &  & \multicolumn{5}{c}{$\widehat{R}$} & \multicolumn{5}{c}{ESS} & \multicolumn{5}{c}{Divergences} & \multicolumn{5}{c}{Time (s)} \\
\cmidrule(lr){3-7} \cmidrule(lr){8-12} \cmidrule(lr){13-17} \cmidrule(lr){18-22}
$\rho$ & Prior & 5 & 10 & 30 & 50 & 100 & 5 & 10 & 30 & 50 & 100 & 5 & 10 & 30 & 50 & 100 & 5 & 10 & 30 & 50 & 100 \\
\midrule
0.3 & R2-GP & \textbf{1.00} & \textbf{1.00} & \textbf{1.01} & \textbf{1.01} & \textbf{1.01} & \textbf{1040} & \textbf{1003} & \textbf{843} & \textbf{645} & 463 & \textbf{0.0} & \textbf{0.0} & \textbf{0.0} & \textbf{0.5} & \textbf{11} & \textbf{6.2} & \textbf{9.1} & \textbf{26} & \textbf{49} & \textbf{124} \\
 & Normal & \textbf{1.00} & 1.01 & \textbf{1.01} & 1.18 & 1.29 & 946 & 901 & 503 & 51 & 20 & 1.4 & 1.7 & 3.9 & 251 & 253 & 8.0 & 15 & 75 & 437 & 1377 \\
 & NG & 1.01 & 1.03 & 1.15 & 1.32 & 1.24 & 524 & 243 & 34 & 17 & 43 & 5.7 & 28 & 95 & 110 & 81 & 12 & 38 & 212 & 356 & 904 \\
 & RHS ($p_0{=}5$) & \textbf{1.00} & 1.01 & \textbf{1.01} & 1.02 & 1.02 & 904 & 832 & 627 & 417 & \textbf{487} & 3.0 & 3.4 & 5.0 & 12 & 21 & 8.6 & 13 & 35 & 65 & 156 \\
\addlinespace
0.5 & R2-GP & \textbf{1.00} & \textbf{1.00} & \textbf{1.00} & \textbf{1.01} & \textbf{1.01} & \textbf{1191} & \textbf{1186} & \textbf{1087} & \textbf{960} & \textbf{648} & \textbf{0.1} & \textbf{0.0} & \textbf{0.0} & \textbf{0.1} & \textbf{1.9} & \textbf{7.3} & \textbf{11} & \textbf{32} & \textbf{53} & \textbf{137} \\
 & Normal & \textbf{1.00} & \textbf{1.00} & 1.01 & 1.02 & 1.51 & 1054 & 1047 & 836 & 250 & 12 & 2.4 & 1.5 & 2.3 & 70 & 222 & 10 & 19 & 87 & 303 & 1331 \\
 & NG & 1.03 & 1.02 & 1.11 & 1.20 & 1.59 & 594 & 293 & 42 & 23 & 19 & 11 & 20 & 121 & 146 & 92 & 13 & 33 & 224 & 435 & 867 \\
 & RHS ($p_0{=}5$) & \textbf{1.00} & 1.01 & 1.01 & \textbf{1.01} & 1.03 & 1010 & 925 & 689 & 625 & 547 & 6.5 & 4.2 & 11 & 16 & 40 & 10 & 16 & 40 & 72 & 169 \\
\addlinespace
0.7 & R2-GP & \textbf{1.00} & \textbf{1.00} & \textbf{1.01} & \textbf{1.01} & \textbf{1.01} & \textbf{1290} & \textbf{1266} & \textbf{1204} & \textbf{1115} & \textbf{811} & \textbf{0.7} & \textbf{0.1} & \textbf{0.1} & \textbf{0.1} & \textbf{1.4} & \textbf{8.3} & \textbf{13} & \textbf{37} & \textbf{64} & \textbf{167} \\
 & Normal & \textbf{1.00} & \textbf{1.00} & \textbf{1.01} & \textbf{1.01} & 1.50 & 1094 & 1108 & 949 & 562 & 11 & 4.0 & 1.7 & 2.7 & 9.2 & 181 & 12 & 26 & 123 & 308 & 1303 \\
 & NG & 1.01 & 1.03 & 1.11 & 1.22 & 1.59 & 674 & 292 & 46 & 22 & 9 & 25 & 53 & 161 & 175 & 89 & 16 & 44 & 260 & 467 & 1047 \\
 & RHS ($p_0{=}5$) & \textbf{1.00} & 1.01 & \textbf{1.01} & \textbf{1.01} & \textbf{1.01} & 1067 & 1027 & 828 & 763 & 612 & 9.1 & 5.8 & 16 & 11 & 21 & 12 & 18 & 50 & 85 & 199 \\
\bottomrule
\end{tabular}
\end{table*}
  \endgroup

\clearpage

  \begingroup
  \simtablesetup{Interval Coverage (\%), Sparse DGP, Mat\'ern $3/2$ Kernel. As in Table~\ref{tab:coverage-sparse-se}.}%
  \begin{table*}[t]
\centering
\footnotesize
\setlength{\tabcolsep}{3pt}
\caption{Coverage (\%) of the nominal $90\%$ posterior interval (sparse DGP, Mat\'ern $3/2$ kernel), averaged over 50 replications. Columns give the number of additive components $D$. For $\alpha$ and $\rho$ it is the share of the active components whose interval contains the truth; the interval of a positive parameter cannot contain the exact zero of an inactive one. Bold marks the coverage closest to $90\%$ in each column within a length-scale. All models use the $N(\rho, 0.1)$ prior centred at the generating length-scale on the length-scales.}\label{tab:coverage-sparse-m32}
\begin{tabular}{ll *{15}{r}}
\toprule
 &  & \multicolumn{5}{c}{$\alpha$} & \multicolumn{5}{c}{$\rho$} & \multicolumn{5}{c}{$\sigma$} \\
\cmidrule(lr){3-7} \cmidrule(lr){8-12} \cmidrule(lr){13-17}
$\rho$ & Prior & 5 & 10 & 30 & 50 & 100 & 5 & 10 & 30 & 50 & 100 & 5 & 10 & 30 & 50 & 100 \\
\midrule
0.3 & R2-GP & 98 & 80 & 27 & 11 & 0 & \textbf{100} & \textbf{100} & \textbf{100} & \textbf{100} & \textbf{100} & \textbf{94} & \textbf{88} & 98 & \textbf{88} & \textbf{94} \\
 & Normal & 95 & \textbf{95} & \textbf{94} & \textbf{97} & \textbf{100} & \textbf{100} & \textbf{100} & \textbf{100} & \textbf{100} & \textbf{100} & 98 & 86 & 2 & 0 & 0 \\
 & NG & 95 & 84 & 70 & 69 & 64 & \textbf{100} & \textbf{100} & \textbf{100} & \textbf{100} & \textbf{100} & 98 & 86 & 86 & 64 & 0 \\
 & RHS ($p_0{=}5$) & \textbf{88} & 68 & 49 & 33 & 24 & \textbf{100} & \textbf{100} & \textbf{100} & \textbf{100} & \textbf{100} & 82 & 86 & \textbf{90} & 66 & 80 \\
\addlinespace
0.5 & R2-GP & 97 & 80 & 18 & 6 & 0 & \textbf{100} & \textbf{100} & \textbf{100} & \textbf{100} & \textbf{100} & \textbf{88} & \textbf{90} & \textbf{88} & \textbf{92} & \textbf{90} \\
 & Normal & 93 & \textbf{93} & \textbf{96} & \textbf{98} & \textbf{100} & \textbf{100} & \textbf{100} & \textbf{100} & \textbf{100} & \textbf{100} & 84 & 86 & 26 & 0 & 0 \\
 & NG & 94 & 85 & 60 & 59 & 62 & \textbf{100} & \textbf{100} & \textbf{100} & \textbf{100} & \textbf{100} & 84 & 86 & \textbf{92} & 78 & 18 \\
 & RHS ($p_0{=}5$) & \textbf{88} & 70 & 38 & 33 & 30 & \textbf{100} & \textbf{100} & \textbf{100} & \textbf{100} & \textbf{100} & 86 & 88 & 94 & 82 & 82 \\
\addlinespace
0.7 & R2-GP & 98 & 80 & 19 & 7 & 0 & \textbf{100} & \textbf{100} & \textbf{100} & \textbf{100} & \textbf{100} & \textbf{88} & \textbf{86} & \textbf{96} & 96 & \textbf{86} \\
 & Normal & 94 & 93 & \textbf{96} & \textbf{96} & \textbf{100} & \textbf{100} & \textbf{100} & \textbf{100} & \textbf{100} & \textbf{100} & \textbf{88} & 82 & 40 & 0 & 0 \\
 & NG & 94 & \textbf{88} & 59 & 58 & 64 & \textbf{100} & \textbf{100} & \textbf{100} & \textbf{100} & \textbf{100} & \textbf{88} & 84 & \textbf{96} & \textbf{90} & 22 \\
 & RHS ($p_0{=}5$) & \textbf{87} & 71 & 40 & 34 & 32 & \textbf{100} & \textbf{100} & \textbf{100} & \textbf{100} & \textbf{100} & 84 & 84 & \textbf{96} & 80 & 82 \\
\bottomrule
\end{tabular}
\end{table*}
  \endgroup
\simtablefitted{Relative RMSE, Sparse DGP, Mat\'ern $3/2$ Kernel. As in Table~\ref{tab:rmse-sparse-se}.}{tables/sparse/m32/parameter_recovery}
  \begingroup
  \simtablesetup{Sampling Diagnostics, Sparse DGP, Mat\'ern $3/2$ Kernel. As in Table~\ref{tab:sampling-dense-se}.}%
  \begin{table*}[t]
\centering
\scriptsize
\setlength{\tabcolsep}{2.5pt}
\caption{Sampling behaviour (sparse DGP, Mat\'ern $3/2$ kernel), each statistic averaged over 50 replications. Columns give the number of additive components $D$. $\widehat{R}$ is the largest and ESS the smallest bulk effective sample size over all $\alpha_d$, $\rho_d$, $\sigma$ and $\mathrm{lp}$ of a fit; divergences are out of 4000 post-warmup draws; time is the sampling time in seconds. Bold marks the best value: the $\widehat{R}$ closest to one, the largest ESS, the fewest divergences and the shortest time, in each column within a length-scale. All models use the $N(\rho, 0.1)$ prior centred at the generating length-scale on the length-scales.}\label{tab:sampling-sparse-m32}
\begin{tabular}{ll *{20}{r}}
\toprule
 &  & \multicolumn{5}{c}{$\widehat{R}$} & \multicolumn{5}{c}{ESS} & \multicolumn{5}{c}{Divergences} & \multicolumn{5}{c}{Time (s)} \\
\cmidrule(lr){3-7} \cmidrule(lr){8-12} \cmidrule(lr){13-17} \cmidrule(lr){18-22}
$\rho$ & Prior & 5 & 10 & 30 & 50 & 100 & 5 & 10 & 30 & 50 & 100 & 5 & 10 & 30 & 50 & 100 & 5 & 10 & 30 & 50 & 100 \\
\midrule
0.3 & R2-GP & \textbf{1.00} & \textbf{1.00} & \textbf{1.01} & \textbf{1.01} & \textbf{1.01} & \textbf{1049} & \textbf{1029} & \textbf{759} & \textbf{555} & \textbf{426} & \textbf{0.0} & \textbf{0.0} & \textbf{0.1} & \textbf{4.1} & \textbf{6.2} & \textbf{7.9} & \textbf{14} & \textbf{45} & \textbf{86} & \textbf{282} \\
 & Normal & \textbf{1.00} & 1.01 & 1.06 & 1.28 & 1.37 & 925 & 840 & 209 & 25 & 14 & 0.5 & 0.6 & 59 & 248 & 233 & 9.8 & 21 & 188 & 847 & 3109 \\
 & NG & 1.01 & 1.03 & 1.15 & 1.36 & 1.15 & 518 & 240 & 40 & 23 & 41 & 2.8 & 29 & 61 & 95 & 77 & 16 & 53 & 312 & 559 & 1934 \\
 & RHS ($p_0{=}5$) & 1.01 & 1.01 & \textbf{1.01} & \textbf{1.01} & 1.03 & 849 & 750 & 458 & 475 & 416 & 2.0 & 0.9 & 16 & 11 & 41 & 11 & 18 & 58 & 107 & 345 \\
\addlinespace
0.5 & R2-GP & \textbf{1.00} & \textbf{1.00} & \textbf{1.01} & \textbf{1.01} & \textbf{1.01} & \textbf{1181} & \textbf{1156} & \textbf{1020} & \textbf{750} & \textbf{483} & \textbf{0.1} & \textbf{0.0} & \textbf{0.1} & \textbf{0.1} & \textbf{6.0} & \textbf{9.1} & \textbf{15} & \textbf{48} & \textbf{89} & 324 \\
 & Normal & \textbf{1.00} & \textbf{1.00} & \textbf{1.01} & 1.19 & 1.33 & 1041 & 1000 & 576 & 40 & 18 & 0.8 & 1.4 & 9.2 & 201 & 223 & 12 & 25 & 146 & 738 & 2664 \\
 & NG & 1.01 & 1.04 & 1.14 & 1.24 & 1.32 & 599 & 271 & 47 & 22 & 35 & 5.6 & 35 & 114 & 151 & 88 & 18 & 53 & 333 & 654 & 1844 \\
 & RHS ($p_0{=}5$) & \textbf{1.00} & 1.01 & \textbf{1.01} & 1.02 & 1.02 & 984 & 888 & 704 & 542 & 419 & 2.3 & 3.5 & 9.1 & 28 & 39 & 13 & 23 & 60 & 107 & \textbf{323} \\
\addlinespace
0.7 & R2-GP & \textbf{1.00} & \textbf{1.00} & \textbf{1.01} & \textbf{1.01} & \textbf{1.01} & \textbf{1239} & \textbf{1239} & \textbf{1121} & \textbf{904} & \textbf{581} & \textbf{0.5} & \textbf{0.0} & \textbf{0.1} & \textbf{0.4} & \textbf{3.4} & \textbf{11} & \textbf{18} & \textbf{53} & \textbf{101} & \textbf{372} \\
 & Normal & \textbf{1.00} & 1.01 & \textbf{1.01} & 1.07 & 1.37 & 1069 & 1057 & 782 & 112 & 17 & 2.1 & 1.0 & 2.4 & 170 & 204 & 15 & 32 & 162 & 666 & 2537 \\
 & NG & 1.01 & 1.02 & 1.12 & 1.20 & 1.41 & 693 & 310 & 43 & 25 & 21 & 17 & 40 & 106 & 165 & 108 & 22 & 65 & 382 & 720 & 1967 \\
 & RHS ($p_0{=}5$) & \textbf{1.00} & 1.01 & \textbf{1.01} & \textbf{1.01} & 1.03 & 1014 & 964 & 748 & 633 & 463 & 5.1 & 5.7 & 8.2 & 15 & 43 & 15 & 25 & 68 & 130 & 466 \\
\bottomrule
\end{tabular}
\end{table*}
  \endgroup

\clearpage

  \begingroup
  \simtablesetup{Interval Coverage (\%), Sparse DGP, Mat\'ern $5/2$ Kernel. As in Table~\ref{tab:coverage-sparse-se}.}%
  \begin{table*}[t]
\centering
\footnotesize
\setlength{\tabcolsep}{3pt}
\caption{Coverage (\%) of the nominal $90\%$ posterior interval (sparse DGP, Mat\'ern $5/2$ kernel), averaged over 50 replications. Columns give the number of additive components $D$. For $\alpha$ and $\rho$ it is the share of the active components whose interval contains the truth; the interval of a positive parameter cannot contain the exact zero of an inactive one. Bold marks the coverage closest to $90\%$ in each column within a length-scale. All models use the $N(\rho, 0.1)$ prior centred at the generating length-scale on the length-scales.}\label{tab:coverage-sparse-m52}
\begin{tabular}{ll *{15}{r}}
\toprule
 &  & \multicolumn{5}{c}{$\alpha$} & \multicolumn{5}{c}{$\rho$} & \multicolumn{5}{c}{$\sigma$} \\
\cmidrule(lr){3-7} \cmidrule(lr){8-12} \cmidrule(lr){13-17}
$\rho$ & Prior & 5 & 10 & 30 & 50 & 100 & 5 & 10 & 30 & 50 & 100 & 5 & 10 & 30 & 50 & 100 \\
\midrule
0.3 & R2-GP & 92 & 82 & 23 & 14 & 1 & \textbf{100} & \textbf{100} & \textbf{100} & \textbf{100} & \textbf{100} & 82 & \textbf{92} & 94 & 98 & \textbf{92} \\
 & Normal & 92 & 96 & \textbf{96} & \textbf{96} & \textbf{100} & \textbf{100} & \textbf{100} & \textbf{100} & \textbf{100} & \textbf{100} & \textbf{88} & 86 & 12 & 0 & 0 \\
 & NG & \textbf{90} & \textbf{86} & 66 & 68 & 65 & \textbf{100} & \textbf{100} & \textbf{100} & \textbf{100} & \textbf{100} & 84 & 84 & 84 & 66 & 2 \\
 & RHS ($p_0{=}5$) & 81 & 67 & 42 & 40 & 26 & \textbf{100} & \textbf{100} & \textbf{100} & \textbf{100} & \textbf{100} & 82 & 86 & \textbf{90} & \textbf{92} & 74 \\
\addlinespace
0.5 & R2-GP & 98 & 74 & 18 & 13 & 0 & \textbf{100} & \textbf{100} & \textbf{100} & \textbf{100} & \textbf{100} & \textbf{90} & \textbf{92} & 94 & \textbf{90} & \textbf{86} \\
 & Normal & 94 & \textbf{93} & \textbf{97} & \textbf{92} & \textbf{100} & \textbf{100} & \textbf{100} & \textbf{100} & \textbf{100} & \textbf{100} & 88 & \textbf{88} & 46 & 4 & 0 \\
 & NG & 94 & 84 & 63 & 64 & 64 & \textbf{100} & \textbf{100} & \textbf{100} & \textbf{100} & \textbf{100} & 86 & \textbf{92} & \textbf{90} & 86 & 10 \\
 & RHS ($p_0{=}5$) & \textbf{88} & 66 & 43 & 44 & 28 & \textbf{100} & \textbf{100} & \textbf{100} & \textbf{100} & \textbf{100} & 88 & \textbf{88} & 80 & 94 & 78 \\
\addlinespace
0.7 & R2-GP & 98 & 76 & 19 & 10 & 0 & \textbf{100} & \textbf{100} & \textbf{100} & \textbf{100} & \textbf{100} & 88 & 94 & 94 & \textbf{90} & 86 \\
 & Normal & 93 & 94 & \textbf{95} & \textbf{93} & \textbf{99} & \textbf{100} & \textbf{100} & \textbf{100} & \textbf{100} & \textbf{100} & \textbf{90} & \textbf{90} & 52 & 10 & 0 \\
 & NG & 94 & \textbf{87} & 62 & 61 & 70 & \textbf{100} & \textbf{100} & \textbf{100} & \textbf{100} & \textbf{100} & 88 & \textbf{90} & \textbf{88} & 86 & 34 \\
 & RHS ($p_0{=}5$) & \textbf{89} & 70 & 42 & 45 & 32 & \textbf{100} & \textbf{100} & \textbf{100} & \textbf{100} & \textbf{100} & 86 & 92 & 86 & \textbf{90} & \textbf{88} \\
\bottomrule
\end{tabular}
\end{table*}
  \endgroup
\simtablefitted{Relative RMSE, Sparse DGP, Mat\'ern $5/2$ Kernel. As in Table~\ref{tab:rmse-sparse-se}.}{tables/sparse/m52/parameter_recovery}
  \begingroup
  \simtablesetup{Sampling Diagnostics, Sparse DGP, Mat\'ern $5/2$ Kernel. As in Table~\ref{tab:sampling-dense-se}.}%
  \begin{table*}[t]
\centering
\scriptsize
\setlength{\tabcolsep}{2.5pt}
\caption{Sampling behaviour (sparse DGP, Mat\'ern $5/2$ kernel), each statistic averaged over 50 replications. Columns give the number of additive components $D$. $\widehat{R}$ is the largest and ESS the smallest bulk effective sample size over all $\alpha_d$, $\rho_d$, $\sigma$ and $\mathrm{lp}$ of a fit; divergences are out of 4000 post-warmup draws; time is the sampling time in seconds. Bold marks the best value: the $\widehat{R}$ closest to one, the largest ESS, the fewest divergences and the shortest time, in each column within a length-scale. All models use the $N(\rho, 0.1)$ prior centred at the generating length-scale on the length-scales.}\label{tab:sampling-sparse-m52}
\begin{tabular}{ll *{20}{r}}
\toprule
 &  & \multicolumn{5}{c}{$\widehat{R}$} & \multicolumn{5}{c}{ESS} & \multicolumn{5}{c}{Divergences} & \multicolumn{5}{c}{Time (s)} \\
\cmidrule(lr){3-7} \cmidrule(lr){8-12} \cmidrule(lr){13-17} \cmidrule(lr){18-22}
$\rho$ & Prior & 5 & 10 & 30 & 50 & 100 & 5 & 10 & 30 & 50 & 100 & 5 & 10 & 30 & 50 & 100 & 5 & 10 & 30 & 50 & 100 \\
\midrule
0.3 & R2-GP & \textbf{1.00} & \textbf{1.00} & \textbf{1.01} & \textbf{1.01} & \textbf{1.02} & \textbf{1033} & \textbf{1032} & \textbf{745} & \textbf{582} & \textbf{386} & \textbf{0.0} & \textbf{0.0} & \textbf{0.1} & \textbf{2.1} & \textbf{20} & \textbf{7.6} & \textbf{12} & \textbf{37} & \textbf{75} & \textbf{234} \\
 & Normal & \textbf{1.00} & 1.01 & 1.02 & 1.22 & 1.34 & 930 & 867 & 368 & 30 & 18 & 1.0 & 0.5 & 32 & 228 & 231 & 9.7 & 20 & 130 & 768 & 2202 \\
 & NG & 1.01 & 1.04 & 1.14 & 1.31 & 1.13 & 536 & 239 & 42 & 23 & 46 & 10 & 13 & 65 & 81 & 43 & 20 & 47 & 304 & 484 & 1628 \\
 & RHS ($p_0{=}5$) & 1.01 & 1.01 & \textbf{1.01} & 1.02 & \textbf{1.02} & 879 & 768 & 482 & 432 & 376 & 2.2 & 1.8 & 9.7 & 14 & 26 & 11 & 18 & 48 & 95 & 295 \\
\addlinespace
0.5 & R2-GP & \textbf{1.00} & \textbf{1.00} & \textbf{1.01} & \textbf{1.01} & \textbf{1.01} & \textbf{1163} & \textbf{1171} & \textbf{1036} & \textbf{812} & \textbf{570} & \textbf{0.2} & \textbf{0.0} & \textbf{0.1} & \textbf{0.3} & \textbf{12} & \textbf{8.9} & \textbf{14} & \textbf{41} & \textbf{81} & \textbf{236} \\
 & Normal & \textbf{1.00} & 1.01 & \textbf{1.01} & 1.16 & 1.38 & 1004 & 1026 & 706 & 77 & 17 & 1.8 & 1.3 & 2.7 & 129 & 189 & 13 & 25 & 131 & 606 & 2350 \\
 & NG & 1.01 & 1.04 & 1.14 & 1.22 & 1.39 & 623 & 233 & 41 & 20 & 25 & 6.7 & 43 & 137 & 134 & 74 & 17 & 53 & 364 & 590 & 1465 \\
 & RHS ($p_0{=}5$) & \textbf{1.00} & 1.01 & \textbf{1.01} & \textbf{1.01} & 1.02 & 977 & 912 & 695 & 548 & 425 & 2.6 & 5.5 & 7.9 & 15 & 31 & 13 & 19 & 56 & 110 & 323 \\
\addlinespace
0.7 & R2-GP & \textbf{1.00} & \textbf{1.00} & \textbf{1.01} & \textbf{1.01} & \textbf{1.01} & \textbf{1234} & \textbf{1262} & \textbf{1168} & \textbf{1006} & \textbf{701} & \textbf{0.2} & \textbf{0.0} & \textbf{0.1} & \textbf{0.1} & \textbf{1.1} & \textbf{11} & \textbf{18} & \textbf{51} & \textbf{95} & \textbf{311} \\
 & Normal & \textbf{1.00} & 1.01 & \textbf{1.01} & 1.05 & 1.56 & 1069 & 1083 & 898 & 298 & 12 & 5.0 & 2.1 & 1.4 & 48 & 185 & 15 & 30 & 161 & 573 & 2116 \\
 & NG & 1.01 & 1.05 & 1.11 & 1.19 & 1.65 & 653 & 270 & 44 & 25 & 14 & 13 & 51 & 164 & 146 & 56 & 19 & 67 & 380 & 716 & 1724 \\
 & RHS ($p_0{=}5$) & \textbf{1.00} & 1.01 & \textbf{1.01} & \textbf{1.01} & 1.02 & 1010 & 1018 & 827 & 666 & 496 & 7.0 & 7.9 & 9.2 & 14 & 30 & 15 & 25 & 64 & 131 & 374 \\
\bottomrule
\end{tabular}
\end{table*}
  \endgroup

\clearpage

\subsection{Non-normalised Competing Priors}
\label{app:sim-unnorm}

All comparisons above apply the competing priors to the normalised components \eqref{eq:fd}. Here we place the normal, NG and RHS priors on $\alpha_d$ in the usual HSGP parameterisation, $f_d = \alpha_d g_d$, in which $\alpha_d^2$ is the marginal variance. R2-GP remains the reference, and all other settings are unchanged. Since $\alpha_d$ is then on a different scale from the generating value, we compare predictive performance. Figures~\ref{fig:unnorm-elpd-dense-se} to \ref{fig:unnorm-elpd-sparse-m52} show that the qualitative behaviour of the competing priors is the same as with the normalisation.

In the dense DGP, R2-GP attains the highest mean elpd in each condition against every non-normalised prior. At $D = 100$ the mean $\Delta\mathrm{elpd}$ lies between $-149$ and $-105$ for the normal prior, between $-47$ and $-33$ for NG, and between $-9.1$ and $-4.5$ for RHS. Since $Q_d < 1$ for the kernels considered, a given prior on $\alpha_d$ implies less signal variance without the normalisation. This lessens the overestimation under the normal and NG priors, whose mean elpd improves by up to $37$ and $24$, but does not remove it. For RHS the mean elpd is lower by at most $1.4$ in both DGPs. In the sparse DGP, RHS again attains a higher mean elpd than R2-GP for $D \ge 30$, by at most $12.1$, and the normal prior again underperforms, with a mean $\Delta\mathrm{elpd}$ between $-150$ and $-106$ at $D = 100$. At $D = 100$ the bias in $\Remp$ has the same sign and a similar range for each prior as with the normalisation, and the normal and NG priors again fail to converge (not shown).

\begin{figure}[!htbp]
  \centering
  \resizebox{0.85\textwidth}{!}{\input{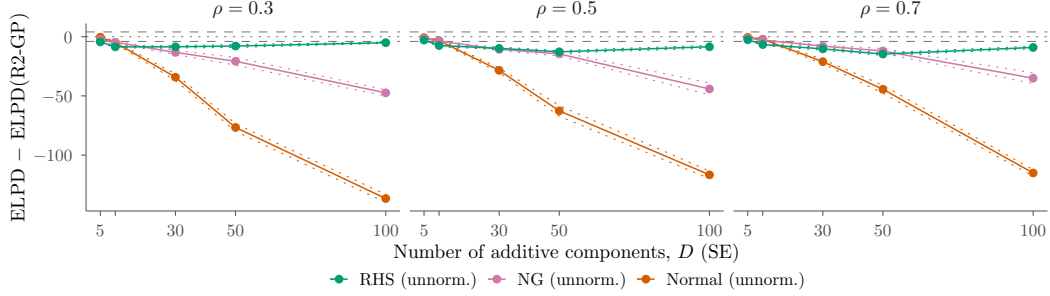}}
  \caption{Non-normalised Competing Priors, Dense DGP, SE Kernel. As in Figure~\ref{fig:dense-elpd-se}, with the normal, NG and RHS priors on $\alpha_d$ in $f_d = \alpha_d g_d$.}
  \label{fig:unnorm-elpd-dense-se}
\end{figure}

\begin{figure}[!htbp]
  \centering
  \resizebox{0.85\textwidth}{!}{\input{figures/tikz/appendix/unnormalised/elpd_difference_unnormalised_paper_dense_m32}}
  \caption{Non-normalised Competing Priors, Dense DGP, Mat\'ern $3/2$ Kernel. As in Figure~\ref{fig:unnorm-elpd-dense-se}.}
  \label{fig:unnorm-elpd-dense-m32}
\end{figure}

\begin{figure}[p]
  \centering
  \resizebox{0.85\textwidth}{!}{\input{figures/tikz/appendix/unnormalised/elpd_difference_unnormalised_paper_dense_m52}}
  \caption{Non-normalised Competing Priors, Dense DGP, Mat\'ern $5/2$ Kernel. As in Figure~\ref{fig:unnorm-elpd-dense-se}.}
  \label{fig:unnorm-elpd-dense-m52}
\end{figure}

\begin{figure}[p]
  \centering
  \resizebox{0.85\textwidth}{!}{\input{figures/tikz/appendix/unnormalised/elpd_difference_unnormalised_paper_sparse_se}}
  \caption{Non-normalised Competing Priors, Sparse DGP, SE Kernel. As in Figure~\ref{fig:unnorm-elpd-dense-se}.}
  \label{fig:unnorm-elpd-sparse-se}
\end{figure}

\begin{figure}[p]
  \centering
  \resizebox{0.85\textwidth}{!}{\input{figures/tikz/appendix/unnormalised/elpd_difference_unnormalised_paper_sparse_m32}}
  \caption{Non-normalised Competing Priors, Sparse DGP, Mat\'ern $3/2$ Kernel. As in Figure~\ref{fig:unnorm-elpd-dense-se}.}
  \label{fig:unnorm-elpd-sparse-m32}
\end{figure}

\begin{figure}[p]
  \centering
  \resizebox{0.85\textwidth}{!}{\input{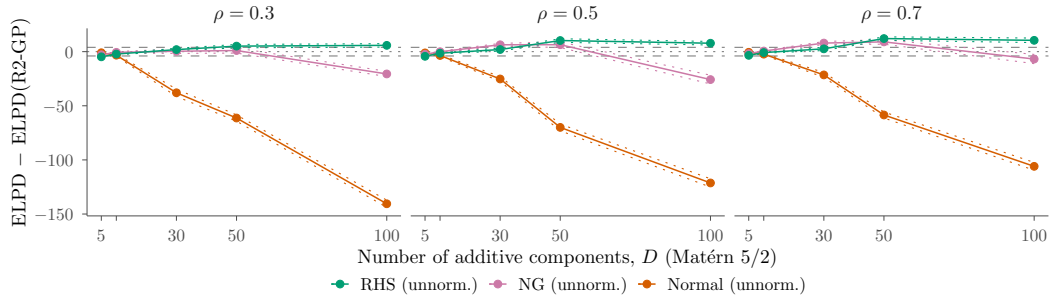}}
  \caption{Non-normalised Competing Priors, Sparse DGP, Mat\'ern $5/2$ Kernel. As in Figure~\ref{fig:unnorm-elpd-dense-se}.}
  \label{fig:unnorm-elpd-sparse-m52}
\end{figure}

\clearpage

\section{Additional Case Study Results}
\label{app:add-case-study}

Figures~\ref{fig:case-study-12effects-se} and \ref{fig:case-study-12effects-m32} extend Figure~\ref{fig:case-study-2effects} to $12$ covariates, for the SE and the Mat\'ern $3/2$ kernel. The covariates are ranked by their posterior mean share of explained variance under R2-GP with the SE kernel, and both figures show the same covariates. For these figures, the models of Section~\ref{sec:case_study} are fitted to all $208$ compounds. Table~\ref{tab:casestudy-sample-diag} gives the predictive performance and the sampling diagnostics over the $10$ train/test splits.

\begin{figure}[!htbp]
  \centering
  \input{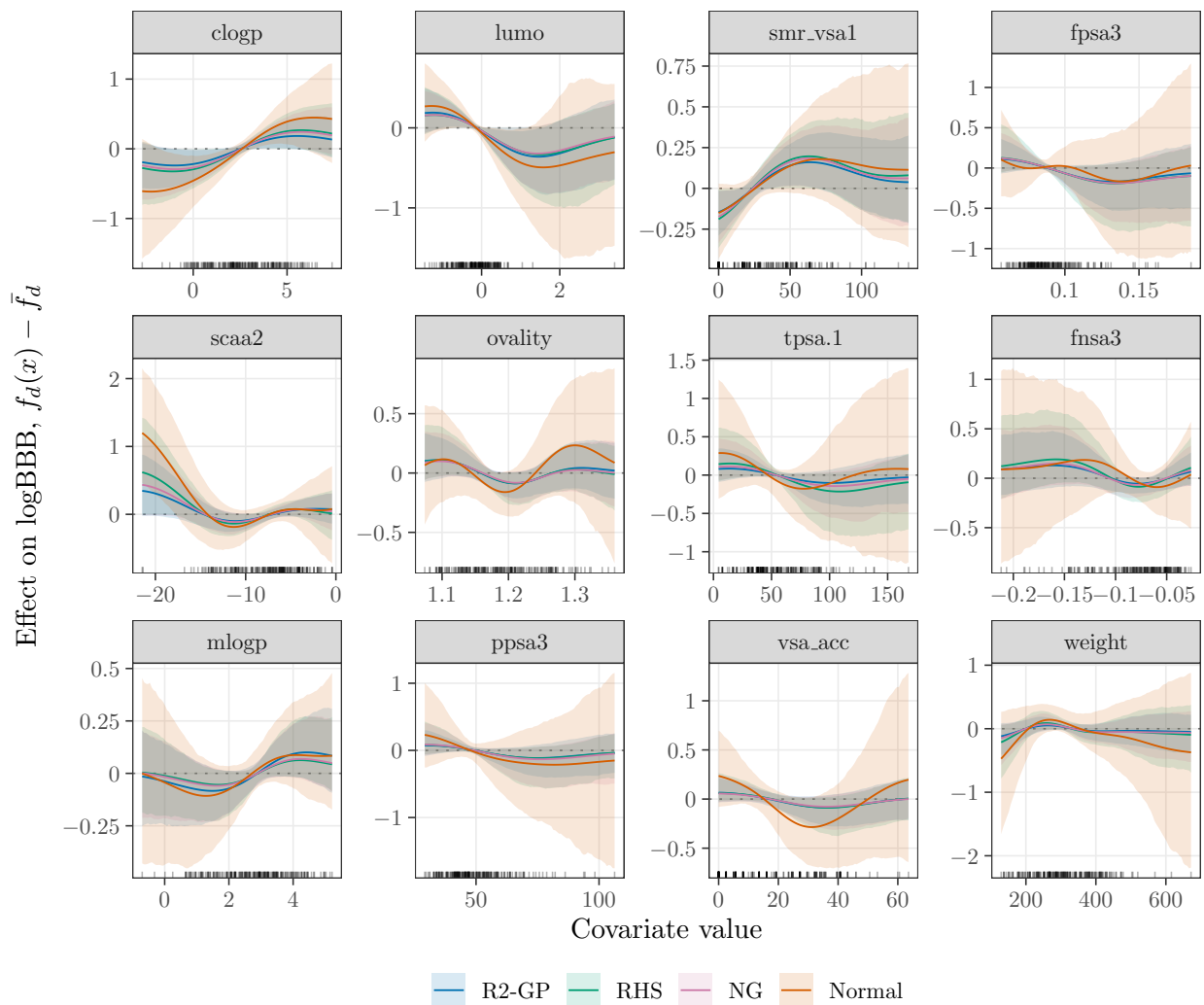}
  \caption{Covariate Effects, BloodBrain Data, SE Kernel. Posterior means (lines) and pointwise $90\%$ intervals (bands) of the centred effects $f_d(x) - \bar{f}_d$ for the $12$ covariates with the largest share of explained variance. Ticks mark the observed covariate values.}
  \label{fig:case-study-12effects-se}
\end{figure}

\begin{figure}[!htbp]
  \centering
  \input{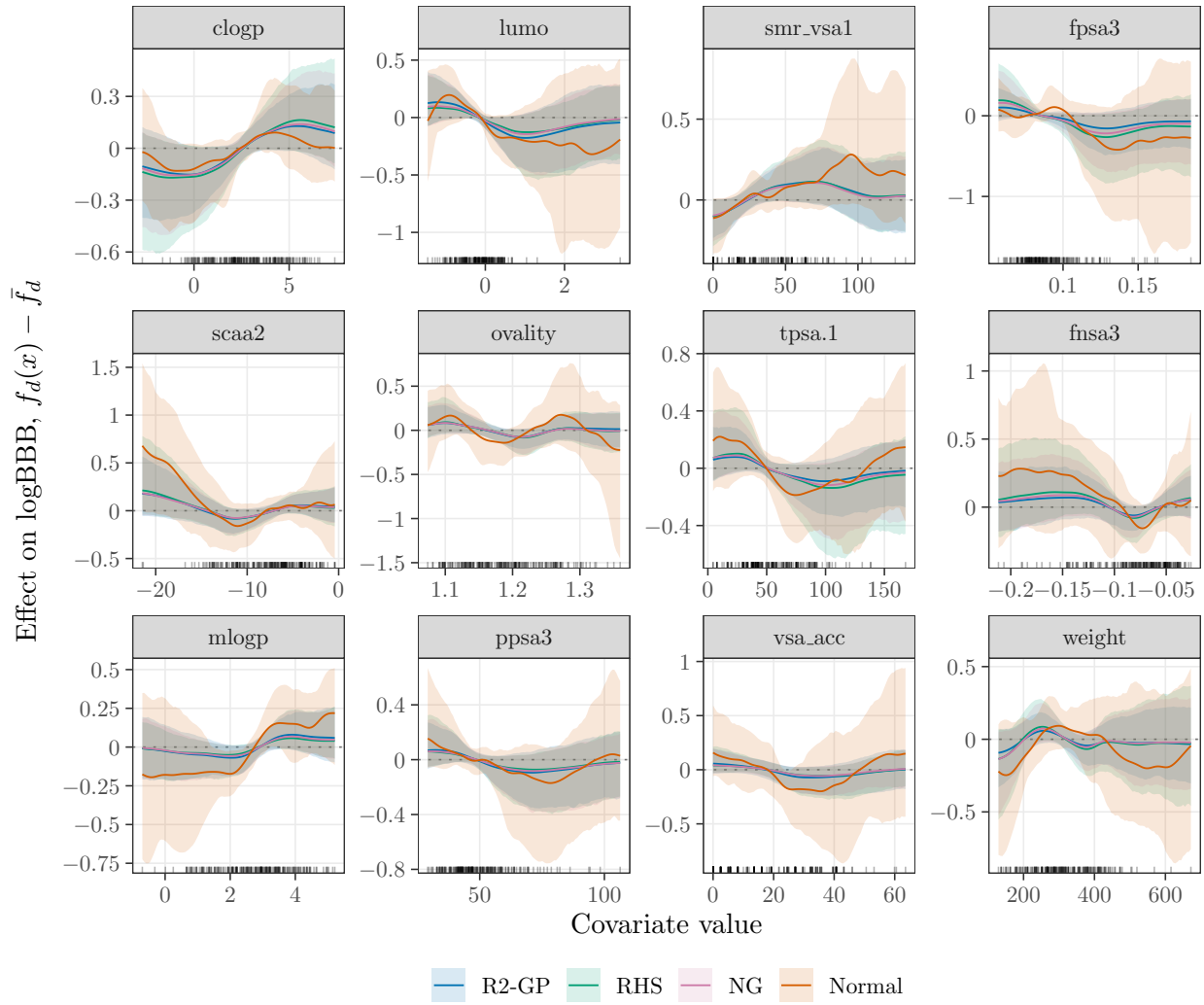}
  \caption{Covariate Effects, BloodBrain Data, Mat\'ern $3/2$ Kernel. As in Figure~\ref{fig:case-study-12effects-se}, for the same covariates.}
  \label{fig:case-study-12effects-m32}
\end{figure}

\begin{table}[!htbp]
\apptable
\small
\setlength{\tabcolsep}{4pt}
\begin{threeparttable}
\caption{Case Study Predictive and Sampling Diagnostics}\label{tab:casestudy-sample-diag}
\begin{tabular}{lrccrrrrrccrrrr}
\toprule
& \multicolumn{7}{c}{SE} & \multicolumn{7}{c}{Mat\'ern 3/2} \\
\cmidrule(lr){2-8} \cmidrule(lr){9-15}
Prior & ELPD & $R^2$ & $R^2_{\mathrm{oos}}$ & $\widehat{R}$ & ESS & Div. & Time & ELPD & $R^2$ & $R^2_{\mathrm{oos}}$ & $\widehat{R}$ & ESS & Div. & Time \\
\midrule
R2-GP & $-40.9$ & 0.73 & 0.56 & 1.008 & 529 & 0.0 & 51 & $-40.2$ & 0.77 & 0.56 & 1.016 & 479 & 0.0 & 126 \\
RHS & $-41.4$ & 0.72 & 0.55 & 1.023 & 178 & 0.2 & 72 & $-41.0$ & 0.76 & 0.55 & 1.023 & 150 & 0.0 & 176 \\
NG & $-41.3$ & 0.70 & 0.55 & 1.232 & 14 & 0.0 & 61 & $-40.7$ & 0.74 & 0.55 & 1.037 & 71 & 0.0 & 126 \\
Normal & $-46.0$ & 0.95 & 0.56 & 2.117 & 5 & 2.5 & 100 & $-48.9$ & 1.00 & 0.55 & 3.462 & 4 & 23.8 & 257 \\
\bottomrule
\end{tabular}
\begin{tablenotes}
    \item Note: Predictive performance and sampling behaviour of fitted models on the BloodBrain data over 10 random train/test split replications. ELPD and posterior $R^2$ (in-sample and out-of-sample) are reported as means over the replications. $\widehat{R}$ and ESS are reported as the worst case over all model parameters across replications; Div.\ is the mean number of divergent transitions per fit and time the median sampling time per fit in minutes.
\end{tablenotes}
\end{threeparttable}
\end{table}

\clearpage

\section{Computing Infrastructure}
\label{app:compute}

All computations ran on CPUs; no GPUs were used. The simulation study was run
on an institutional high-performance computing cluster managed by Slurm, on
nodes with Intel Xeon (Haswell to Cascade Lake) and AMD EPYC (Milan)
processors. Each fit was allocated four cores, one per chain, and 3\,GB of
memory, and the workflow was orchestrated with Snakemake~8.18. Models were
fitted with CmdStan~2.35.0 \citep{carpenter2017stan} through the R package
\textit{cmdstanr}~0.9.0 in R~4.5.3, compiled with GCC~12.3.

The simulation comprises $57\,750$ model fits, all of which completed.
Sampling took 1.3 minutes per fit at the median and 5.9 minutes on average;
$99\%$ of the fits finished within about an hour and the slowest took 4.7
hours. In total this is about $5\,700$ hours of four-core time, or $23\,000$
core-hours. The case study adds $80$ fits on the data splits, which took
between 27 minutes and 4.5 hours each (median 1.7 hours), about $650$
core-hours in total, and eight fits to the full data.

\end{document}